\documentclass{article}

\usepackage{microtype}
\usepackage{graphicx}
\usepackage{subfigure}
\usepackage{booktabs} 
\usepackage[utf8]{inputenc} 
\usepackage[T1]{fontenc}    
\usepackage{url}            
\usepackage{amsfonts, dsfont}       
\usepackage{amsmath}
\usepackage{amsthm}
\usepackage{nicefrac}       
\usepackage{bm,bbm}
\usepackage{color}
\usepackage{tikz}
\usepackage{wrapfig}
\usepackage{lmodern}
\usetikzlibrary{fit,calc,positioning,decorations.pathreplacing}
\usetikzlibrary{shapes,arrows,calc,matrix}
\usepackage{enumitem}
\usepackage{breakcites}
\usepackage{titling}
\usepackage{multirow}

\newtheorem{theorem}{Theorem}[section]
\newtheorem{corollary}{Corollary}[theorem]
\newtheorem{lemma}[theorem]{Lemma}

\usepackage{hyperref}

\hypersetup{
    colorlinks,
    linkcolor={black},
    citecolor={black},
    urlcolor={black},
}

\newcommand{\w}{\mathbf{w}}
\newcommand{\smlw}{\mathbf{w}}

\newcommand{\smlh}{\mathbf{h}}
\newcommand{\smla}{\mathbf{a}}
\newcommand{\smlc}{\mathbf{c}}

\newcommand{\x}{\mathbf{x}}
\newcommand{\z}{\mathbf{z}}
\newcommand{\y}{\mathbf{y}}
\newcommand{\e}{\mathbf{e}}

\newcommand{\Bigh}{\mathbf{H}}
\newcommand{\Bigx}{\mathbf{X}}
\newcommand{\Bigy}{\mathbf{Y}}
\newcommand{\A}{\mathbf{A}}

\newcommand{\one}{\mathds{1}}
\newcommand{\boldmu}{\pmb{\mu}}
\newcommand{\boldeps}{\pmb{\epsilon}}
\newcommand{\prox}{\mathcal{S}}

\newcommand{\R}{{\mathbb{R}}}

\newcommand{\W}{\mathbf{W}}

\newcommand{\bvec}{\mathbf{b}}
\newcommand{\cvec}{\mathbf{c}}
\newcommand{\eye}{\mathbf{I}}

\newcommand{\vectornorm}[1]{\left|\left|#1\right|\right|}
\newcommand{\ip}[2]{\langle #1,#2 \rangle}

\DeclareMathOperator*{\argmax}{arg\,max}
\DeclareMathOperator*{\argmin}{arg\,min}

\usepackage[accepted]{icml2020}

\icmltitlerunning{Convolutional dictionary learning based auto-encoders for natural exponential-family distributions}

\begin{document}

\twocolumn[
\icmltitle{Convolutional dictionary learning based auto-encoders\\for natural exponential-family distributions}

\icmlsetsymbol{equal}{*}

\begin{icmlauthorlist}
\icmlauthor{Bahareh Tolooshams}{equal,harvard}
\icmlauthor{Andrew H. Song}{equal,mit}
\icmlauthor{Simona Temereanca}{brown}
\icmlauthor{Demba Ba}{harvard}
\end{icmlauthorlist}

\icmlaffiliation{harvard}{School of Engineering and Applied Sciences, Harvard University, Cambridge, MA}
\icmlaffiliation{mit}{Massachusetts Institute of Technology, Cambridge, MA}
\icmlaffiliation{brown}{Brown University, Boston, MA}

\icmlcorrespondingauthor{Bahareh Tolooshams}{btolooshams@seas.harvard.edu}

\icmlkeywords{Auto-Encoders, Exponential-Family Distributions, Dictionary Learning, Image Denoising}

\vskip 0.3in
]

\printAffiliationsAndNotice{\icmlEqualContribution} 

\begin{abstract}
	
We introduce a class of auto-encoder neural networks tailored to data from the natural exponential family (e.g., count data). The architectures are inspired by the problem of learning the filters in a convolutional generative model with sparsity constraints, often referred to as convolutional dictionary learning (CDL). Our work is the first to combine ideas from convolutional generative models and deep learning for data that are naturally modeled with a non-Gaussian distribution (e.g., binomial and Poisson). This perspective provides us with a scalable and flexible framework that can be re-purposed for a wide range of tasks and assumptions on the generative model. Specifically, the iterative optimization procedure for solving CDL, an unsupervised task, is mapped to an unfolded and constrained neural network, with iterative adjustments to the inputs to account for the generative distribution. We also show that the framework can easily be extended for discriminative training, appropriate for a supervised task. We demonstrate 1) that fitting the generative model to learn, in an unsupervised fashion, the latent stimulus that underlies neural spiking data leads to better goodness-of-fit compared to other baselines, 2) competitive performance compared to state-of-the-art algorithms for supervised Poisson image denoising, with significantly fewer parameters, and 3) gradient dynamics of shallow binomial auto-encoder.

\end{abstract}

\section{Introduction}
\label{sec:intro}

Learning shift-invariant patterns from a dataset has given rise to work in different communities, most notably in signal processing (SP) and deep learning. In the former, this problem is referred to as convolutional dictionary learning (CDL) \cite{garcia-2018-convolutional}. CDL imposes a linear \textit{generative model} where the data are generated by a sparse linear combination of shifts of localized patterns. In the latter, convolutional neural networks (NNs) \cite{dl} have excelled in identifying shift-invariant patterns. 

Recently, the iterative nature of the optimization algorithms for performing CDL has inspired the utilization of NNs as an efficient and scalable alternative, starting with the seminal work of \cite{LISTA}, and followed by~\cite{TolooshamsBahareh2018SCDL, SreterHillel2018LCSC, sulam2019multi}. Specifically, the iterative steps are expressed as a recurrent NN, and thus solving the optimization simply becomes passing the input through an unrolled NN~\cite{HersheyJohnR2014DUMI,monga2019algorithm}. At one end of the spectrum, this perspective, through weight-tying, leads to architectures with significantly fewer parameters than a generic NN. At the other end, by untying the weights, it motivates new architectures that depart, and could not be arrived at, from the generative perspective~\cite{LISTA,SreterHillel2018LCSC}.

The majority of the literature at the intersection of generative models and NNs assumes that the data are real-valued and therefore are not appropriate for binary or count-valued data, such as neural spiking data and photon-based images~\cite{yang2011bits}. Nevertheless, several works on Poisson image denoising, arguably the most popular application involving non real-valued data, can be found separately in both communities.
In the SP community, the negative Poisson data likelihood is either explicitly minimized~\cite{Salmon2014, Giryes2014} or used as a penalty term added to the objective of an image denoising problem with Gaussian noise~\cite{Ma2013}. Being rooted in the dictionary learning formalism, these methods operate in an \textit{unsupervised} manner. Although they yield good denoising performance, their main drawbacks are scalability and computational efficiency.

In the deep learning community, NNs tailored to image denoising~\cite{dncnn, Remez2018, Feng18}, which are reminiscent of residual learning, have shown great performance on Poisson image denoising. However, since these 1) are not designed from the generative model perspective and/or 2) are \textit{supervised} learning frameworks, it is unclear how they can be adapted to the classical CDL, where the task is \textit{unsupervised} and the interpretability of the parameters is important. NNs with a generative flavor, namely variational auto-encoders (VAEs), have been extended to utilize non real-valued data \cite{Nazabal2018VAE,liang2018VAE}. However, these architectures cannot be adapted to solve the CDL task.

To address this gap, we make the following contributions\footnote{The code can be found at https://github.com/ds2p/dea}:

\textbf{Auto-encoder inspired by CDL for non real-valued data} We introduce a flexible class of auto-encoder (AE) architectures for data from the natural exponential-family that combines the perspectives of generative models and NNs. We term this framework, depicted in Fig.~\ref{fig:pipeline}, the deep convolutional exponential-family auto-encoder (DCEA). 

\textbf{Unsupervised learning of convolutional patterns} We show through simulation that DCEA performs CDL and learns convolutional patterns from binomial observations. We also apply DCEA to real neural spiking data and show that it fits the data better than baselines.

\textbf{Supervised learning framework} DCEA, when trained in a supervised manner, achieves similar performance to state-of-the-art algorithms for Poisson image denoising with orders of magnitude fewer parameters compared to other baselines, owing to its design based on a generative model.

\textbf{Gradient dynamics of shallow exponential auto-encoder} Given some assumptions on the binomial generative model with dense dictionary and ``good'' initializations, we prove in Theorem~\ref{theo:descentrelu} that shallow exponential auto-encoder (SEA), when trained by gradient descent, recovers the dictionary.

\section{Problem Formulation}
\paragraph{Natural exponential-family distribution}
For a given observation vector $\mathbf{y}\in\mathbb{R}^N$, with mean $\pmb{\mu}\in \mathbb{R}^N$, we define the log-likelihood of the \textit{natural exponential family}~\cite{glm} as
\begin{equation}\label{eq:likelihood}
\log p(\mathbf{y}\vert \bm{\mu}) = f\big(\boldmu\big)^{\text{T}} \y+g(\mathbf{y})-B\big(\boldmu\big),
\end{equation}
where we have assumed that, conditioned on $\boldmu$, the elements of $\y$ are independent. The natural exponential family includes a broad family of probability distributions such as the Gaussian, binomial, and Poisson. The functions $g(\cdot)$, $B(\cdot)$, as well as the invertible \emph{link function} $f(\cdot)$, all depend on the choice of distribution. 

\paragraph{Convolutional generative model}
We assume that $f(\boldmu)$ is the sum of scaled and time-shifted copies of $C$ finite-length filters (dictionary) $\{\mathbf{h}_c\}_{c=1}^C\in\mathbb{R}^K$, each localized, i.e., $K$ $\ll$ $N$. We can express $f(\boldmu)$ in a convolutional form: $f(\pmb{\mu})=\sum_{c=1}^C \smlh_c\ast\x^c$, where $\ast$ is the convolution operation, and $\x^c\in\mathbb{R}^{N-K+1}$ is a train of scaled impulses which we refer to as \emph{code vector}. Using linear-algebraic notation, $f(\pmb{\mu})=\sum_{c=1}^C \smlh_c\ast\x^c=\Bigh\x$, where $\mathbf{H}\in \mathbb{R}^{N\times C(N-K+1)}$ is a matrix that is the concatenation of $C$ Toeplitz (i.e., banded circulant) matrices $\mathbf{H}^c\in \mathbb{R}^{N\times (N-K+1)}$, $c=1,\ldots,C$, and $\mathbf{x}=[(\mathbf{x}^1)^{\text{T}},\ldots,(\mathbf{x}^C)^{\text{T}}]^\text{T}\in\mathbb{R}^{C(N-K+1)}$.

We refer to the input/output domain of $f(\cdot)$ as the data and dictionary domains, respectively. We interpret $\y$ as a time-series and the non-zero elements of $\x$ as the times when each of the $C$ filters are active. When $\y$ is two-dimensional (2D), i.e., an image, $\x$ encodes the spatial locations where the filters contribute to its mean $\boldmu$.

\paragraph{Exponential convolutional dictionary learning (ECDL)}
Given $J$ observations $\{\mathbf{y}^j \}_{j=1}^J$, we estimate $\{\mathbf{h}_c\}_{c=1}^C$ and $\{\mathbf{x}^j\}_{j=1}^J$ that minimize the negative log-likelihood $\sum_{j=1}^Jl(\x^j)=-\sum_{j=1}^J\log p(\y^j\vert \{\smlh_c\}_{c=1}^C,\x^j)$ under the convolutional generative model, subject to sparsity constraints on $\{\mathbf{x}^j\}_{j=1}^J$. We enforce sparsity using the $\ell_1$ norm, which leads to the non-convex optimization problem
\begin{equation}\label{eq:gopt}
\min_{\substack{\{\mathbf{h}_c\}_{c=1}^C\\ \{\x^j\}_{j=1}^J}} \sum_{j=1}^J \overbrace{-(\Bigh\x^j)^{\text{T}} \mathbf{y}^j+B(f^{-1}\big(\Bigh\x^j)\big)}^{l(\x^j)}+\lambda\|\x^j\|_1,
\end{equation}
where the regularizer $\lambda$ controls the degree of sparsity. A popular approach to deal with the non-convexity is to minimize the objective over one set of variables, while the others are fixed, in an alternating manner, until convergence~\cite{Agarwal2016LearningSU}. When $\{\x^{j}\}_{j=1}^J$ is being optimized with fixed $\Bigh$, we refer to the problem as convolutional sparse coding (CSC). When $\Bigh$ is being optimized with $\{\x^{j}\}_{j=1}^J$ fixed, we refer to the problem as convolutional dictionary update (CDU).

\begin{figure}[htb]
	\begin{minipage}[b]{1.0\linewidth}
		\centering
		\tikzstyle{block} = [draw, fill=none, rectangle, 
		minimum height=2em, minimum width=2em]
		\tikzstyle{sum} = [draw, fill=none, minimum height=0.1em, minimum width=0.1em, circle, node distance=1cm]
		\tikzstyle{cir} = [draw, fill=none, circle, line width=0.7mm, minimum width=0.5cm, node distance=1cm]
		\tikzstyle{loss} = [draw, fill=none, color=black, ellipse, line width=0.5mm, minimum width=0.7cm, node distance=1cm]
		\tikzstyle{blueloss} = [draw, fill=none, color=black, ellipse, line width=0.5mm, minimum width=0.7cm, node distance=1cm, color=black]
		\tikzstyle{input} = [coordinate]
		\tikzstyle{output} = [coordinate]
		\tikzstyle{pinstyle} = [pin edge={to-,thin,black}]
		\begin{tikzpicture}[auto, node distance=2cm,>=latex']
		cloud/.style={
			draw=red,
			thick,
			ellipse,
			fill=none,
			minimum height=1em}
		\node [input, name=input] {};
		\node [cir, node distance=1.cm, right of=input] (Y) {$\y^j$};
		\node [cir, right of=Y, node distance=1.3cm] (ytilde) {$\tilde \y_{t}^j$};
		
		\node [block, right of=ytilde,  minimum width=0.5cm, node distance=1.23cm] (HT) {$\alpha\Bigh^{\text{T}}$};
		\node [block, right of=HT,  minimum width=0.7cm, node distance=1.2cm] (relu) {$\prox_{b}$};
		
		\node [cir, right of=relu, node distance=1.1cm] (xt) {$\x_{t}^j$};
		
		\node [fill=none, below of=xt, node distance=0.53cm] (connection_1) {$$};
		\node [fill=none, below of=xt,left=3.6pt, node distance=0.53cm] (connection_2) {$$};
			
		\node [cir, right of=xt, node distance=1.33cm] (xT) {$\x_{T}^j$};
		\node [output, node distance=1.2cm, right of=xT] (output) {};
		\node [block, right of=xT, node distance=1.1cm] (H) {$\Bigh$};
		\node [output, right of=H, node distance=0.6cm] (out) {};
		
		\node [rectangle, below of=H, minimum width=0.01cm, node distance=0.651cm, left=0.000cm] (Y1m) {};
		\node [rectangle, below of=H, minimum width=0.01cm, node distance=0.58cm, left=0.006cm] (Y1ml) {};
		\node [rectangle, below of=H, minimum width=0.01cm, node distance=0.58cm, right=0.006cm] (Y1mr) {};
		
		\node [block, below of=xt, node distance=1.0cm] (H_cns) {$\Bigh$};
		\node [block, left of=H_cns, node distance=1.5cm] (sigmoid) {$f^{-1}(\cdot)$};
		\node [rectangle, fill=none,  node distance=0.6cm,  below of=H] (middle) {};
			
		\draw[thick, line width=2, black, ->]     ($(xt.north east)+(0.55,0.8)$) -- ($(Y.north east)+(7.4,0.8)$);
		\draw[thick, line width=2, black, ->]     ($(Y.north east)+(0,0.8)$) -- ($(xt.north east)+(0.4,0.8)$);
		
		\draw[thick,dotted]     ($(xT.north east)+(-1.06,0.25)$) rectangle ($(ytilde.south west)+(-0.47,-1.25)$); 
		\node [rectangle, fill=none,  node distance=1.32cm,  right=-15pt,above of=H] (text) {\footnotesize{Decoder}};
		\node [rectangle, fill=none,  node distance=1.32cm,  right=0pt,above of=HT] (encoder) {\footnotesize{Encoder}};
			
		\draw [->] (Y) -- node [name=m, pos=0.3, above] {} (ytilde);
		\draw [->] (ytilde) -- node {} (HT);
		\draw [->] (HT) -- node[name=s, pos=0.3, above] {} (relu);
		\draw [->] (relu) -- node[] {} (xt);
		\draw [-] (xt) |- node[] {} (connection_2);
		\draw [->] (connection_1) -| node[] {} (s);
		\draw [->] (xt) -- node[name=loop, pos=0.28, above] {} (xT);
		\draw [->] (xT) -- node[] {} (H);
		\draw [->] (loop) |- node[] {} (H_cns);
		\draw [->] (H_cns) -- node[] {} (sigmoid);
		\draw [->] (sigmoid) -| node[pos=1,right] {} (m);
		\draw [->] (sigmoid) -| node[pos=0.9,right] {-} (m);
		
		\draw [->] (H) -- node[] {} (out);	
		
		\node [rectangle, fill=none,  node distance=0.75cm,  above of=HT] (text) {\footnotesize{Repeat $T$ times}};	
		\end{tikzpicture}
	\end{minipage}
	\caption{DCEA architecture for ECDL. The encoder/decoder structure mimics the CSC/CDU sequence in CDL. The encoder performs $T$ iterations of CSC. Each iteration uses \textit{working observations} (residuals) $\widetilde \y_t^j$ obtained by iteratively modifying $\y^j$ using the filters $\Bigh$ and a nonlinearity $f^{-1}(\cdot)$ that depends on the distribution of $\y^j$. The dictionary $\Bigh$ is updated through the backward pass.}
	\label{fig:pipeline}
\end{figure}
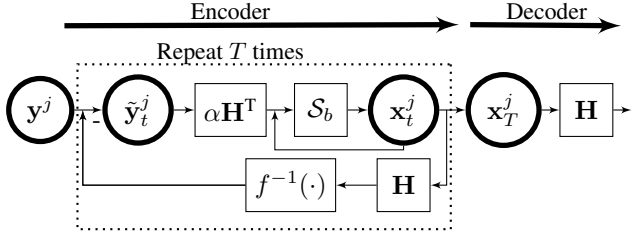
\section{Deep convolutional exponential-family auto-encoder}
We propose a class of auto-encoder architectures to solve the ECDL problem, which we term deep convolutional exponential-family auto-encoder (DCEA). Specifically, we make a one-to-one connection between the CSC/CDU steps and the encoder/decoder of DCEA depicted in Fig.~\ref{fig:pipeline}. We focus only on CSC for a single $\y^j$, as the CSC step can be parallelized across examples.

\subsection{The architecture}

\paragraph{Encoder} The forward pass of the encoder maps the input $\y^{j}$ into the sparse code $\x^j$. Given the filters $\{\smlh_c\}_{c=1}^C$, the encoder solves the $\ell_1$-regularized optimization problem from Eq.~(\ref{eq:gopt})
\begin{equation}\label{eq:l1csc}
\min_{\substack{\x^j}\\} l(\x^j) + \lambda \|\x^j\|_1
\end{equation}
in an iterative manner by unfolding $T$ iterations of the proximal gradient algorithm~\cite{parikh2014}.
For Gaussian observations, Eq.~(\ref{eq:l1csc}) becomes an $\ell_1$-regularized least squares problem, for which several works have unfolded the proximal iteration into a recurrent network~\cite{LISTA, wang2015deep,SreterHillel2018LCSC,TolooshamsBahareh2019deepresidualAE}.

We use $\prox_{b}\in\{\text{ReLU}_{b}, \text{Shrinkage}_{b}\}$ to denote a proximal operator with bias $b\geq 0$. We consider three operators,
\begin{equation}
\begin{aligned}
\text{ReLU}_{b}(\z)&=(\z-b)\cdot\mathds{1}_{\{\z\geq b\}}\\
\text{Shrinkage}_{b}(\z)&=\text{ReLU}_{b}(\z) - \text{ReLU}_{b}(-\z),
\end{aligned}
\end{equation} 
where $\mathds{1}$ is an indicator function. If we constrain the entries of $\x^j$ to be non-negative, we use $\prox_{b}=\text{ReLU}_{b}$. Otherwise, we use $\prox_{b}=\text{Shrinkage}_{b}$. A single iteration of the proximal gradient step is given by
\begin{equation}\label{eq:fista}
\begin{aligned}
\x_t^j &= \prox_{b}\left(\x^j_{t-1} - \alpha \nabla_{\mathbf{x}^j_{t-1}}\log{p\big(\mathbf{y}^j}\vert \{\smlh_c \}_{c=1}^C,\x_{t-1}^j\big)\right)\\
&= \prox_{b}\big(\x_{t-1}^j + \alpha \Bigh^{\text{T}}(\underbrace{\y-f^{-1}\big(\Bigh\x_{t-1}^j\big)}_{\widetilde{\y}_t^j})\big),
\end{aligned}
\end{equation}
where $\x_t^j$ denotes the sparse code after $t$ iterations of unfolding, and $\alpha$ is the step size of the gradient update. The term $\widetilde{\y}_t^j$ is referred to as \textit{working observation}.
The choice of $\alpha$, which we explore next, depends on the generative distribution. We also note that there is a one-to-one mapping between the regularization parameter $\lambda$ and the bias $b$ of $\prox_b$. We treat $\lambda$, and therefore $b$, as hyperparameters that we tune to the desired sparsity level. The matrix $\Bigh^{\text{T}}$ effectively computes the correlation between $\widetilde{\y}_t^j$ and $\{\smlh_c\}_{c=1}^C$. Assuming that we unfold $T$ times, the output of the encoder is $\x_T^j$.

The architecture consists of two nonlinear activation functions: $\prox_{b}(\cdot)$ to enforce sparsity, and $f^{-1}(\cdot)$, the inverse of the link function. For Gaussian observations $f^{-1}(\cdot)$ is linear with slope $1$. For other distributions in the natural exponential family, the encoder uses $f^{-1}(\cdot)$, a mapping from the dictionary domain to the data domain, to transform the input $\y^j$ at each iteration into a working observation $\widetilde\y^j_t$.

\paragraph{Decoder \& training} We apply the decoder $\Bigh$ to $\x_T^j$ to obtain the linear predictor $\Bigh \x_T^j$. This decoder completes the forward pass of DCEA, also summarized in Algorithm 2 of the \textbf{Appendix}. We use the negative log-likelihood $\mathcal{L}_{\Bigh}^{\text{unsup.}}=\sum_{j=1}^J l(\x^j)$ as the loss function applied to the decoder output for updating the dictionary. We train the weights of DCEA, fully specified by the filters  $\{\smlh_c\}_{c=1}^C$, by gradient descent through backpropagation. Note that the $\ell_1$ penalty is not a function of $\Bigh$ and is not in the loss function.

\begin{table*}[htb]
	\centering
	\caption{Generative models for DCEA.}
	\begin{tabular}{l||c|c|c|c|c}
		\toprule
		& $\y^j$ & $f^{-1}(\cdot)$ & $B(\z)$ & $\widetilde{\y}_t^j$ & $\x_t^j$\\
		\hline
		Gaussian & $\mathbb{R}$ & $I(\cdot)$ & $\z^{\text{T}}\z$ & $\y^j-\Bigh \x_{t-1}^j$& $\prox_{b}\left(\x^j_{t-1} + \alpha \Bigh^{\text{T}}\widetilde{\y}_t^{j}\right)$\\
		Binomial & $[0..M_j]$ & sigmoid$(\cdot)$ & $-\mathbf{1}^{\text{T}} \log(\mathbf{1}-\z)$&$\y^{j}-M_j\cdot \text{sigmoid}(\Bigh\x^j_{t-1})$ & $\prox_{b}\left(\x^j_{t-1} + \alpha \Bigh^{\text{T}}(\frac{1}{M_j}\widetilde{\y}_t^{j})\right)$\\
		Poisson & $[0..\infty)$&  $\exp(\cdot)$ & $\mathbf{1}^{\text{T}} \z$ &$\y^{j}-\exp(\Bigh \x_{t-1}^j)$& $\prox_{b}\left(\x^j_{t-1} + \alpha \Bigh^{\text{T}}\left(\text{Elu}(\widetilde{\y}_t^{j})\right)\right)$\\
		\bottomrule
	\end{tabular}
	\label{tab:generative}
\end{table*}

\subsection{Binomial and Poisson generative models}
We focus on two representative distributions for the natural exponential family: binomial and Poisson. For the binomial distribution, $\y^j$ assumes integer values from $0$ to $M_j$. For the Poisson distribution, $\y^j$ can, in principle, be any non-negative integer values, although this is rare due to the exponential decay of the likelihood for higher-valued observations. Table~\ref{tab:generative} summarizes the relevant parameters for these distributions.

The fact that binomial and Poisson observations are integer-valued and have limited range, whereas the underlying $\boldmu_j = f^{-1}(\Bigh\x^j)$ is real-valued, makes the ECDL challenging. This is compounded by the nonlinearity of $f^{-1}(\cdot)$, which distorts the error in the \textit{data} domain, when mapped to the \textit{dictionary} domain. In comparison, in Gaussian CDL 1) the observations are real-valued and 2) $f^{-1}(\cdot)$ is linear.

This implies that, for successful ECDL, $\y^j$ needs to assume a diverse set of integer values. For the binomial distribution, this suggests that $M_j$ should be large. For Poisson, as well as binomial, the maximum of $\boldmu_j$ should also be large. This explains why the performance is generally lower in Poisson image denoising for a lower peak, where the peak is defined as the maximum value of $\boldmu_j$~\cite{Giryes2014}.

\paragraph{Practical design considerations for architecture}
As the encoder of DCEA performs iterative proximal gradient steps, we need to ensure that $\x_T^j$ converges. Convergence analysis of ISTA~\cite{beck2009fast} shows that if $l(\x^j)$ is convex and has $L$--Lipschitz continuous gradient, which loosely means the Hessian is upper-bounded everywhere by $L>0$, choosing $\alpha \in(0,1/L]$ guarantees convergence. For the Gaussian distribution, $L$ is the square of the largest singular value of $\Bigh$~\cite{Daubechies2003AnIT}, denoted $\sigma_{\text{max}}^2(\Bigh)$, and therefore $\alpha \in(0,1/\sigma_{\text{max}}^2(\Bigh)]$. For the Binomial distribution, $L=\frac{1}{4}\sigma_{\text{max}}^2(\Bigh)$, and therefore $\alpha \in(0,4/\sigma_{\text{max}}^2(\Bigh)]$.

The gradient of the Poisson likelihood is not Lipschitz continuous. Therefore, we cannot set $\alpha$ a priori. In practice, the step size at every iteration is determined through a back-tracking line search~\cite{boyd_vandenberghe_2004}, a process that is not trivial to replicate in the forward pass of a NN. We observed that the lack of a principled approach for picking $\alpha$ results, at times, in the residual assuming large negative values, leading to instabilities. Therefore, we imposed a finite lower-bound on the residual through an exponential linear unit (Elu)~\cite{elu15} which, empirically, we found to work the best compared to other nonlinear activation units. The convergence properties of this approach require further theoretical analysis that is outside of the scope of this paper.

\section{Connection to unsupervised/supervised paradigm}
We now analyze the connection between DCEA and ECDL. We first examine how the convolutional generative model places constraints on DCEA. 
Then, we provide a theoretical justification for using DCEA for ECDL by proving dictionary recovery guarantees under some assumptions. Finally, we explain how DCEA can be modified for a supervised task.
   
\subsection{Constrained structure of DCEA}
We discuss key points that allow DCEA to perform ECDL.\\\\
\textbf{- Linear 1-layer decoder} In ECDL, the only sensible decoder is a one layer decoder comprising $\Bigh$ and a linear activation. In contrast, the decoder of a typical AE consists of multiple layers along with nonlinear activations.\\
\textbf{- Tied weights across encoder and decoder} Although our encoder is \textit{deep}, the same weights ($\Bigh$ and $\Bigh^{\text{T}}$) are repeated across the layers.\\
\textbf{- Alternating minimization} The forward pass through the encoder performs the CSC step, and the backward pass, via backpropagation, performs the CDU step.

\subsection{Theory for unsupervised ECDL task}
Here, we prove that under some assumptions, when training the shallow exponential-family auto-encoder by approximate gradient descent, the network recovers the dictionary corresponding to the binomial generative model. Consider a SEA, where the encoder is unfolded only once. Recent work~\cite{Nguyen2019} has shown that, using observations from a sparse \emph{linear} model with a \emph{dense} dictionary, the Gaussian SEA trained by approximate gradient descent recovers the dictionary. We prove a similar result for the binomial SEA trained with binomial observations with mean $f(\boldmu) = \A\x^*$, i.e., a sparse \emph{nonlinear} generative model with a \emph{dense} dictionary $\A \in \R^{n \times p}$.
\begin{theorem}\label{theo:descentrelu}
Suppose the generative model satisfies (A1) - (A14) (see {\bf Appendix}). Given infinitely many examples (i.e., $J\rightarrow \infty$), the binomial SEA with $\prox_{\bvec} = \text{ReLU}_{\bvec}$ trained by approximate gradient descent followed by normalization using the learning rate of $\kappa = O(p/s)$ (i.e., $\smlw^{(l+1)}_i = \text{normalize}(\smlw^{(l)}_i - \kappa g_i)$) recovers $\A$. More formally, there exists $\delta$$\in$$(0,1)$ such that at every iteration $l$, $\forall i\ \| \smlw_i^{(l+1)} - \smla_i \|_2^2 \leq (1- \delta) \| \smlw_i^{(l)} - \smla_i \|_2^2 + \kappa \cdot O(\frac{\max(s^2, s^3/ p^{\frac{2}{3} + 2\xi})}{p^{1+6\xi}})$.
\end{theorem}
\noindent Theorem~\ref{theo:descentrelu} shows recovery of $\A$ for the cases when $p$ grows faster than $s$ and the amplitude of the codes are bounded on the order of $O(\frac{1}{p^{\frac{1}{3}+\xi}})$. We refer the reader to the {\bf Appendix} for the implications of this theorem and its interpretation.
\subsection{DCEA as a supervised framework}\label{section:supervised}
For the \textit{supervised} paradigm, given the desired output (i.e., clean image, $\y_{\text{clean}}^j$, in the case of image denoising), we relax the DCEA architecture and untie the weights~\cite{SreterHillel2018LCSC, rethinking, TolooshamsBahareh2019deepresidualAE} as follows
\begin{equation}\label{eq:supervised_proximal}
\x_t^j =\prox_{\bvec}\big(\x_{t-1}^j + \alpha (\mathbf{W}^e)^{\text{T}}(\y-f^{-1}\big(\mathbf{W}^d\x_{t-1}^j\big))\big),
\end{equation}
where we still use $\Bigh$ as the decoder. We use $\{\mathbf{w}_c^e\}_{c=1}^C$ and $\{\mathbf{w}_c^d\}_{c=1}^C$ to denote the filters associated with $\mathbf{W}^{e}$ and $\mathbf{W}^d$, respectively. We train the bias vector $\bvec$, unlike in the unsupervised setting where we tune it by grid search~\cite{TolooshamsBahareh2019deepresidualAE, tasissa2020dense}. Compared to DCEA for ECDL, the number of parameters to learn has increased three-fold.

Although the introduction of additional parameters implies the framework is no longer exactly optimizing the parameters of the convolutional generative model, DCEA still maintains the core principles of the convolutional generative model. First, DCEA performs CSC, as $\mathbf{W}^e,\mathbf{W}^d,\text{ and }\Bigh$ are convolutional matrices and $\prox_{\bvec}$ ensures sparsity of $\x_T^j$. Second, the encoder uses $f^{-1}(\cdot)$, as specified by natural exponential family distributions.
Therefore, we allow only a moderate departure from the generative model to balance the problem formulation and the problem-solving mechanism. Indeed, as we show in the Poisson image denoising of Section~\ref{sec:exp}, the denoising performance for DCEA with untied weights is superior to that of DCEA with tied weights.

Indeed, the constraints can be relaxed further. For instance, 1) the proximal operator $\prox_{\bvec}$ can be replaced by a deep NN~\cite{Mardani2018}, 2) the inverse link function $f^{-1}(\cdot)$ can be replaced by a NN~\cite{Gao2016}, 3) $\textbf{W}^d$, $\textbf{W}^e$, and $\Bigh$ can be untied across different iterations~\cite{HersheyJohnR2014DUMI}, and 4) the linear 1-layer decoder can be replaced with a deep nonlinear decoder. These would increase the number of trainable parameters, allowing for more expressivity and improved performance. Nevertheless, as our goal is to maintain the connection to \textit{sparsity} and the \textit{natural exponential family}, while keeping the number of parameters small, we do not explore these possibilities in this work. 

\section{Experiments}\label{sec:exp}
We apply our framework in three different settings.\\\\
\textbf{- Poisson image denoising} (\textit{supervised}) We evaluate the performance of supervised DCEA in Poisson image denoising and compare it to state-of-the-art algorithms.\\
\textbf{- ECDL for simulation} (\textit{unsupervised}) We use simulations to examine how the unsupervised DCEA performs ECDL for \textit{binomial} data. With access to ground-truth data, we evaluate the accuracy of the learned dictionary. Additionally, we conduct ablation studies in which we relax the constraints on the DCEA architecture and assess how accuracy changes.\\
\textbf{- ECDL for neural spiking data} (\textit{unsupervised}) Using neural spiking data collected from mice~\cite{Temereanca2008}, we perform unsupervised ECDL using DCEA. As is common in the analysis of neural data~\cite{Truccolo2005}, we assume a \textit{binomial} generative model.

\begin{table*}[htb]
	\renewcommand{\arraystretch}{1.3}
	\centering
	\caption{PSNR performance (in dB) of Poisson image denoising for five different models on test images for peak 1, 2, and 4: 1) SPDA, 2) BM3D+VST, 3) Class-agnostic, 4) DCEA constrained (DCEA-C), and 5) DCEA unconstrained (DCEA-UC).}
	\setlength\tabcolsep{4pt}
	\begin{tabular}{|c|c|c|c|c|c|c|c|c|c|c|c|}
		\hline
		 && Man & Couple & Boat & Bridge & Camera & House & Peppers & Set12 & BSD68 & \# of Params\\
		 \hline
		 \multirow{5}{*}{Peak 1}
		 &SPDA &.&.&21.42 &19.20&20.23&22.73&19.99& 20.39 & $\cdot$ &160,000\\
		 &BM3D+VST &21.62&21.14& 21.47 &19.22&20.37&22.35&19.89& $\cdot$ & 21.01 & N/A\\
		 &Class-agnostic &\textbf{22.49}&22.11& \textbf{22.38} & 19.83 & \textbf{21.59} & 22.87 & \textbf{21.43} & \textbf{21.51} & 21.78 &655,544\\
		 &DCEA-C (ours) &22.03&21.76&21.80&19.72&20.68&21.70&20.22&20.72&21.27&20,618\\
		 &DCEA-UC (ours) &22.44&\textbf{22.16}&22.34&\textbf{19.87}&21.47&\textbf{23.00}&20.91&21.37&\textbf{21.84}&61,516\\
		\hline
		\multirow{5}{*}{Peak 2}
		&SPDA &.&.&21.73&20.15&21.54&\textbf{25.09}&21.23&21.70&$\cdot$&160,000\\
		 &BM3D+VST &23.11&22.65& 22.90&20.31&22.13&24.18&21.97&$\cdot$& 22.21 & N/A\\
		 &Class-agnostic &\textbf{23.64}&\textbf{23.30}& \textbf{23.66} & 20.80 & \textbf{23.25} & 24.77& \textbf{23.19} &\textbf{22.97}&22.90&655,544\\
		 &DCEA-C (ours) &23.10&2.79&22.90&20.60&22.01&23.22&21.70&22.02&22.31&20,618\\
		 &DCEA-UC (ours) &23.57&\textbf{23.30}&23.51&\textbf{20.82}&22.94&24.52&22.94&22.79&\textbf{22.92}&61,516\\
		 \hline
		 \multirow{5}{*}{Peak 4}
		 &SPDA &.&.&22.46&20.55&21.90&26.09&22.09&22.56&$\cdot$&160,000\\
		 &BM3D+VST &24.32&24.10&24.16&21.50&23.94&26.04&24.07& $\cdot$ & 23.54 & N/A\\
		 &Class-agnostic &24.77&24.60& 24.86 & 21.81 & \textbf{24.87} & \textbf{26.59} & \textbf{24.83}& \textbf{24.40} & 23.98 &655,544\\
		 &DCEA-C (ours) &24.26&24.08&24.25&21.59&23.60&25.11&23.68&23.51&23.54&20,618\\
		 &DCEA-UC (ours) &\textbf{24.82}&\textbf{24.69}&\textbf{24.89}&\textbf{21.83}&24.66&26.47&24.71&24.37&\textbf{24.10}&61,516\\
		 \hline
	\end{tabular}
	\label{tab:psnrbest}
\end{table*}

\begin{figure*}[!htb]
	\begin{minipage}[b]{1.0\linewidth}
		\centering
		\tikzstyle{input} = [coordinate]
		\tikzstyle{output} = [coordinate]
		\tikzstyle{pinstyle} = [pin edge={to-,thin,black}]
		\begin{tikzpicture}[auto, node distance=2cm,>=latex']
		cloud/.style={
			draw=red,
			thick,
			ellipse,
			fill=none,
			minimum height=1em}

		\node [input, name=input] {};
		
		\node [rectangle, fill=none, node distance=0.001cm, right of=input] (A) {$\includegraphics[width=0.24\linewidth]{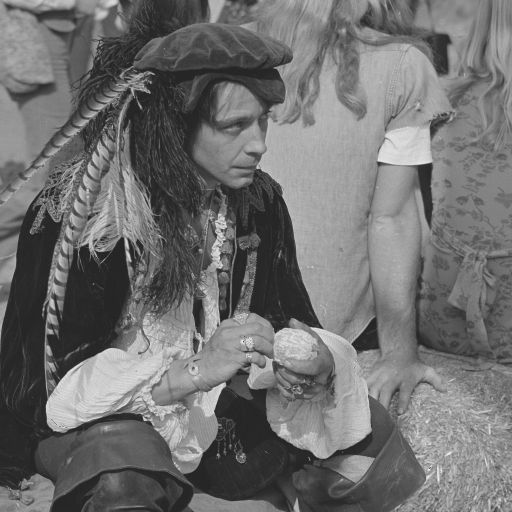}$};
		\node [rectangle, fill=none, node distance=4.15cm, right of=A] (B) {$\includegraphics[width=0.24\linewidth]{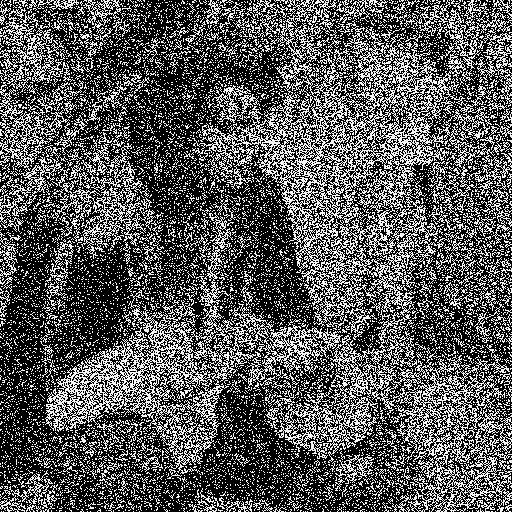}$};
		\node [rectangle, fill=none, node distance=4.15cm, right of=B] (C) {$\includegraphics[width=0.24\linewidth]{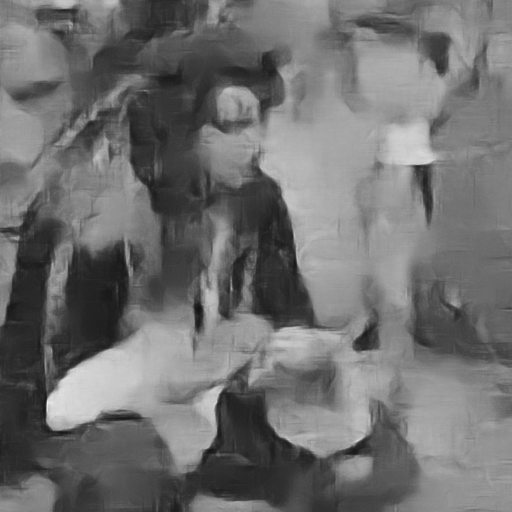}$};
		\node [rectangle, fill=none, node distance=4.15cm, right of=C] (D) {$\includegraphics[width=0.24\linewidth]{./figures/man_untied_1}$};
				
		\node [rectangle, fill=none,  node distance=2.25cm,  above of=A] (text) {(a) Original};
		\node [rectangle, fill=none,  node distance=2.25cm,  above of=B] (text) {(b) Noisy peak$=1$};
		\node [rectangle, fill=none,  node distance=2.25cm,  above of=C] (text) {(c) DCEA-C};
		\node [rectangle, fill=none,  node distance=2.25cm,  above of=D] (text) {(d) DCEA-UC};     
		
		\node [rectangle, fill=none, node distance=4.15cm, below of=A] (A) {$\includegraphics[width=0.24\linewidth]{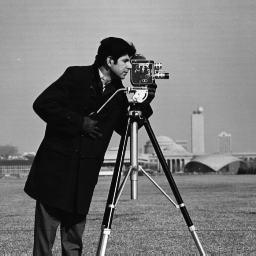}$};
		\node [rectangle, fill=none, node distance=4.15cm, right of=A] (B) {$\includegraphics[width=0.24\linewidth]{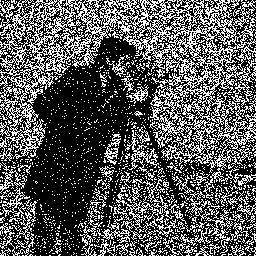}$};
		\node [rectangle, fill=none, node distance=4.15cm, right of=B] (C) {$\includegraphics[width=0.24\linewidth]{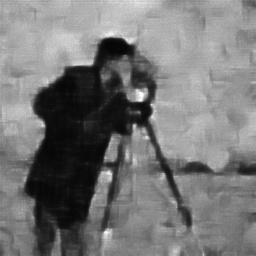}$};
		\node [rectangle, fill=none, node distance=4.15cm, right of=C] (D) {$\includegraphics[width=0.24\linewidth]{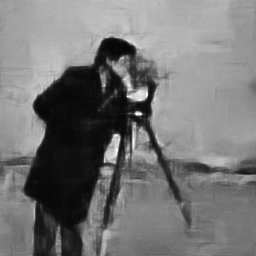}$};
						
		\end{tikzpicture}
	\end{minipage}
	\caption{Denoising performance on test images with peak$=1$. (a) Original, (b) noisy, (c) DCEA-C, and (d) DCEA-UC.}
	\label{fig:vis}
\end{figure*}

\subsection{Denoising Poisson images}
We evaluated the performance of DCEA on Poisson image denoising for various peaks. We used the peak signal-to-noise-ratio (PSNR) as a metric. DCEA is trained in a \textit{supervised} manner on the PASCAL VOC image set~\cite{pascal-voc-2012} containing $J = 5{,}700$ training images. $\prox_{\bvec}$ is set to $\text{ReLU}_{\bvec}$. We used two test datasets: 1) Set12 (12 images) and 2) BSD68 (68 images)~\cite{MartinFTM01}.

\paragraph{Methods} We trained two versions of DCEA to assess whether relaxing the generative model, thus increasing the number of parameters, helps improve the performance: 1) DCEA constrained (DCEA-C), which uses $\Bigh$ as the convolutional filters and 2) DCEA unconstrained (DCEA-UC), which uses $\Bigh$, $\mathbf{W}^e$, and $\mathbf{W}^d$, as suggested in Eq.~(\ref{eq:supervised_proximal}). We used $C=169$ filters of size $11\times 11$, where we used convolutions with strides of 7 and followed a similar approach to~\cite{rethinking} to account for all shifts of the image when reconstructing. In terms of the number of parameters, DCEA-C has $20{,}618 \,(= 169\times 11 \times 11 + 169)$ and DCEA-UC has $61{,}516\,(= 3\times 169\times 11 \times 11 + 169)$, where the last terms refer to the bias $\bvec$. We set $\alpha = 1$.

We unfolded the encoder for $T=15$ iterations. We initialized the filters using draws from a standard Gaussian distribution scaled by $\sqrt{1/L}$, where we approximate $L$ using the iterative power method. 
We used the ADAM optimizer with an initial learning rate of $10^{-3}$, which we decrease by a factor of 0.8 every 25 epochs, and trained the network for 400 epochs. At every iteration, we crop a random $128 \times 128$ patch, $\y_{\text{clean}}^j$, from a training image and normalize it to $\boldmu_{j,\text{clean}} = \y_{\text{clean}}^j/Q^j$, where $Q^j = \max(\y_{\text{clean}}^j)/ \text{peak}$, such that the maximum value of $\boldmu_{j,\text{clean}}$ equals the desired peak. Then, we generate a count-valued Poisson image with rate $\boldmu_{j,\text{clean}}$, i.e., $\y^j\sim\operatorname{Poisson}(\boldmu_{j,\text{clean}})$. We minimized the mean squared error between the clean image, $\y_{\text{clean}}^j$, and its reconstruction, $Q^j \widehat{\boldmu}_j = Q^j \exp(\Bigh \x_T^j)$.

\noindent We compared DCEA against the following baselines. For a fair comparison, we do not use the binning strategy~\cite{Salmon2014} of these methods, as a pre-processing step.\\\\
\textbf{- Sparse Poisson dictionary algorithm (SPDA)} This is a patch-based dictionary learning framework~\cite{Giryes2014}, using the Poisson generative model with the $\ell_0$ pseudo-norm to learn the dictionary in an \textit{unsupervised} manner, for a given noisy image. SPDA uses 400 filters of length 400, which results in $160{,}000$ parameters.\\  
\textbf{- BM3D + VST} BM3D is an image denoising algorithm based on a sparse representation in a transform domain, originally designed for Gaussian noise. This algorithm applies a variance-stabilizing transform (VST) to the Poisson images to make them closer to Gaussian-perturbed images~\cite{vst}.\\
\textbf{- Class-agnostic denoising network (CA)} This is a denoising residual NN for both Gaussian and Poisson images~\cite{Remez2018}, trained in a \textit{supervised} manner.

\begin{figure*}[!htb]
	\begin{minipage}[b]{1.0\linewidth}
		\centering
		\includegraphics[width=\linewidth]{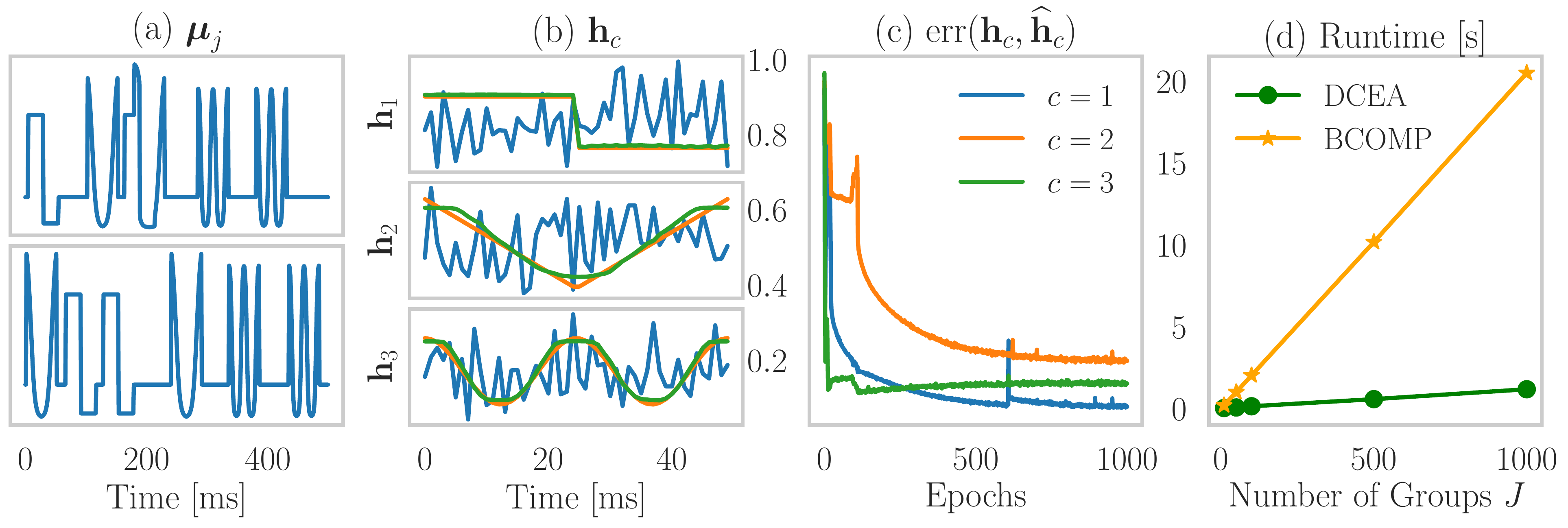}
	\end{minipage}
	\caption{Simulated results with DCEA. (a) Example rate functions, $\boldmu_j$, for two different groups. (b) Initial (blue), true (orange), and learned (green) filters for binomial data. (c) $\text{err}(\smlh_c,\widehat{\smlh}_c)$ over $1,000$ epochs. (d) Total runtime for inference for DCEA and BCOMP.}
	\label{fig:simulation_result}
\end{figure*}

\paragraph{Results} Table~\ref{tab:psnrbest} shows that DCEA outperforms SPDA and BM3D + VST, and shows competitive performance against CA, with an order of magnitude fewer parameters. Fig.~\ref{fig:vis} shows the denoising performance of DCEA-C and DCEA-UC on two test images from Set12 (see \textbf{Appendix} for more examples). We summarize a few additional points from this experiment.

\textbf{- SPDA vs. DCEA-UC} DCEA-UC is significantly more computationally efficient compared to SPDA. SPDA takes several minutes, or hours in some cases, to denoise a single Poisson noisy image whereas, upon training, DCEA performs denoising in less than a second.

\textbf{- CA vs. DCEA-UC} DCEA-UC achieves competitive performance against CA, despite an order of magnitude difference in the number of parameters ($650$K for CA vs. $61$K for DCEA). We conjecture that given the same number of parameters, DCEA-UC would outperform CA. For example, we found that replacing the linear decoder with a nonlinear two-layer decoder ($\approx120$K number of parameters) in DCEA-UC resulted in an increase in PSNR of  $\sim$$0.2$ dB.


\textbf{- DCEA-C vs. DCEA-UC} We also observe that DCEA-UC achieves better performance than DCEA-C. As discussed in Section~\ref{section:supervised}, the relaxation of the generative model, which allows for a three-fold increase in the number of parameters, helps improve the performance.

\subsection{Application to simulated neural spiking data}
\subsubsection{Accuracy of ECDL for DCEA}~\label{section:simulation}
We simulated time-series of neural spiking activity from $J = 1{,}000$ neurons according to the binomial generative model. We used $C=3$ templates of length $K = 50$ and, for each example $j$, generated $f(\boldmu_j)=\Bigh \x^j\in\R^{500}$, where each filter $\{\smlh_c \}_{c=1}^3$ appears five times uniformly random in time. Fig.~\ref{fig:simulation_result}(a) shows an example of two different means, $\boldmu_{j_1}$, $\boldmu_{j_2}$ for $j_1\neq j_2$. Given $\boldmu_j$, we simulated two sets of binary time-series, each with $M_j=25$, $\y^{j} \in \{0,1,\ldots,25\}^{500}$, one of which is used for training and the other for validation.

\noindent \textbf{Methods} For DCEA, we initialized the filters using draws from a standard Gaussian, tuned the regularization parameter $\lambda$ (equivalently $b$ for $\prox_b$) manually, and trained using the unsupervised loss. We place non-negativity constraints on $\x^j$ and thus use $\prox_{b} = \text{ReLU}_{b}$. For baseline, we developed and implemented a method which we refer to as binomial convolutional orthogonal matching pursuit (BCOMP). At present, there does not exist an optimization-based framework for ECDL. Existing dictionary learning methods for non-Gaussian data are patch-based~\cite{Lee09, Giryes2014}. BCOMP combines efficient convolutional greedy pursuit~\cite{Mailhe2011} and binomial greedy pursuit~\cite{Lozano2011}. BCOMP solves Eq.~(\ref{eq:gopt}), but uses $\lVert \x^j \rVert_0$ instead of $\lVert \x^j \rVert_1$. For more details, we refer the reader to the \textbf{Appendix}.

\noindent \textbf{Results} 
Fig.~\ref{fig:simulation_result}(b) demonstrates that DCEA (green) is able to learn $\{\smlh_c\}_{c=1}^3$ accurately. Letting $\{\widehat{\smlh}_c\}_{c=1}^3$ denote the estimates, we quantify the error between a filter and its estimate using the standard measure~\cite{Agarwal2016LearningSU}, $\text{err}(\smlh_c,\widehat{\smlh}_c) = \sqrt{1 - \ip{\smlh_c}{\widehat{\smlh}_c}^2}$, for $\|\smlh_c\|=\|\widehat{\smlh}_c\|=1$. Fig.~\ref{fig:simulation_result}(c) shows the error between the true and learned filters by DCEA, as a function of epochs (we consider all possible permutations and show the one with the lowest error). The fact that the learned and the true filters match demonstrates that DCEA is indeed performing ECDL. Finally, Fig.~\ref{fig:simulation_result}(d) shows the runtime for both DCEA (on GPU) and BCOMP (on CPU) on CSC task, as a function of number of groups $J$, where DCEA is much faster. This shows that DCEA, due to 1) its simple implementation as an unrolled NN and 2) the ease with which the framework can be deployed to GPU, is an efficient/scalable alternative to optimization-based BCOMP.

\subsubsection{Generative model relaxation for ECDL}
Here, we examine whether DCEA with untied weights, which implies a departure from the original convolutional generative model, can still perform ECDL accurately. To this end, we repeat the experiment from Section~\ref{section:simulation} with DCEA-UC, whose parameters are $\Bigh$, $\mathbf{W}^e$, and $\mathbf{W}^d$. Fig.~\ref{fig:mismatch} shows the learned filters, $\widehat{\mathbf{w}}_c^e$, $\widehat{\mathbf{w}}_c^d$, and $\widehat{\mathbf{h}}_c$ for $c=1\text{ and }2$, along with the true filters. For visual clarity, we only show the learned filters for which the distance to the true filters are the closest, among $\widehat{\mathbf{w}}_c^e$, $\widehat{\mathbf{w}}_c^d$, and $\widehat{\mathbf{h}}_c$. We observe that none of them match the true filters. In fact, the error between the learned and the true filters are bigger than the initial error.

This is in sharp contrast to the results of DCEA-C (Fig.~\ref{fig:simulation_result}(b)), where $\Bigh=\mathbf{W}^e=\mathbf{W}^d$. This shows that, to accurately perform ECDL, the NN architecture needs to be strictly constrained such that it optimizes the objective formulated from the convolutional generative model.  

\begin{figure}[htb]
	\centering
	\includegraphics[width=\linewidth]{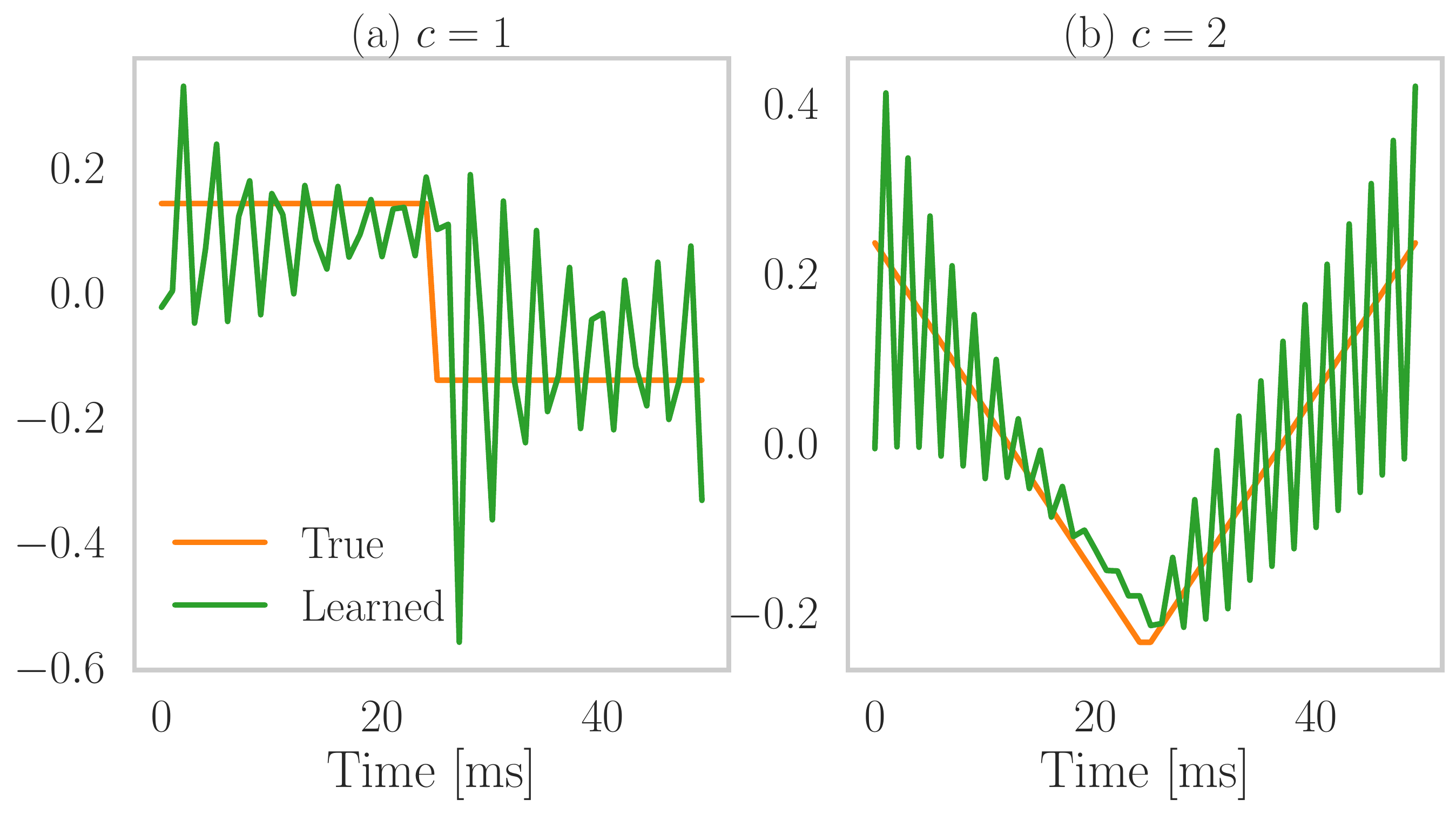}
	\caption{The learned (green) and true (orange) filters for DCEA-UC, when the weights are untied.}
	\label{fig:mismatch}
\end{figure}

\begin{figure}[!ht]
	\centering
	\includegraphics[width=\linewidth]{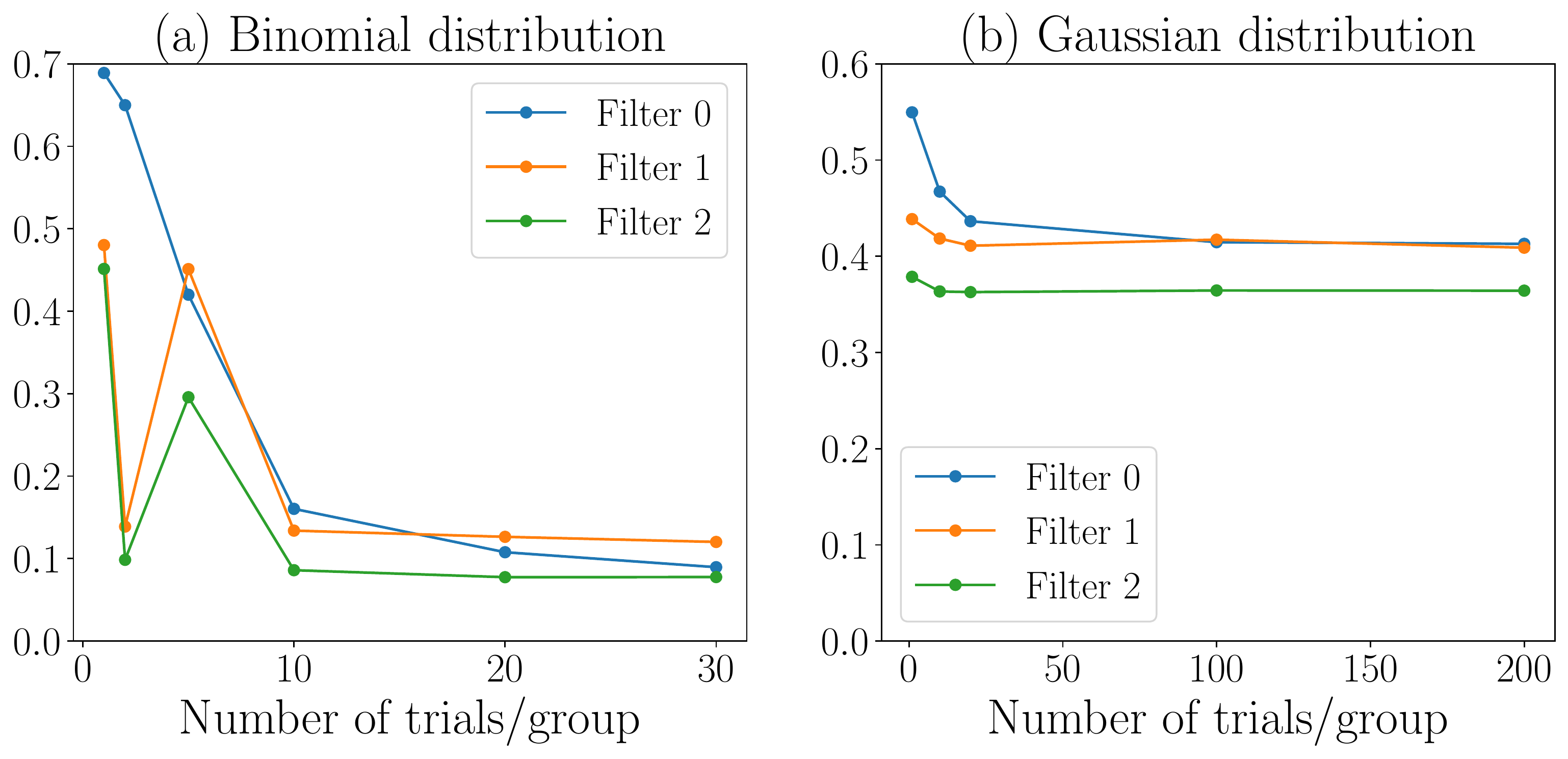}
	\caption{Dictionary error, $\text{err}(\smlh_c,\widehat{\smlh}_c)$, as a function of number of trials $M_j$ per group for the (a) Binomial and (b) Gaussian models. Each point represents the median of 20 independent trials. }
	\label{fig:binomial_vs_gaussian}
\end{figure}

\subsubsection{Effect of model mis-specification on ECDL}
Here, we examine how model mis-specification in DCEA, equivalent to mis-specifying 1) the loss function (negative log-likelihood) and 2) the nonlinearity $f^{-1}(\cdot)$, affects the accuracy of ECDL. We trained two models: 1) DCEA with sigmoid link and binomial likelihood (DCEA-b), the correct model for this experiment, and 2) DCEA with linear link and Gaussian likelihood (DCEA-g). Fig.~\ref{fig:binomial_vs_gaussian} shows how the error $\text{err}(\smlh_c,\widehat{\smlh}_c)$, at convergence, changes as a function of the number of observations $M_j$.

We found that DCEA-b successfully recovers dictionaries for large $M_j$ (>15). Not surprisingly, as $M_j$, i.e. SNR, decreases, the error increases. DCEA-g with 200 observations achieves an error close to 0.4, which is significantly worse than the 0.09 error of DCEA-b with $M_j=30$. These results highlight the importance, for successful dictionary learning, of specifying an appropriate model. The framework we propose, DCEA, provides a flexible inference engine that can accommodate a variety of data-generating models in a seamless manner.

\begin{figure*}[!htb]
	\begin{minipage}[b]{1.0\linewidth}
		\centering
		\includegraphics[width=\linewidth]{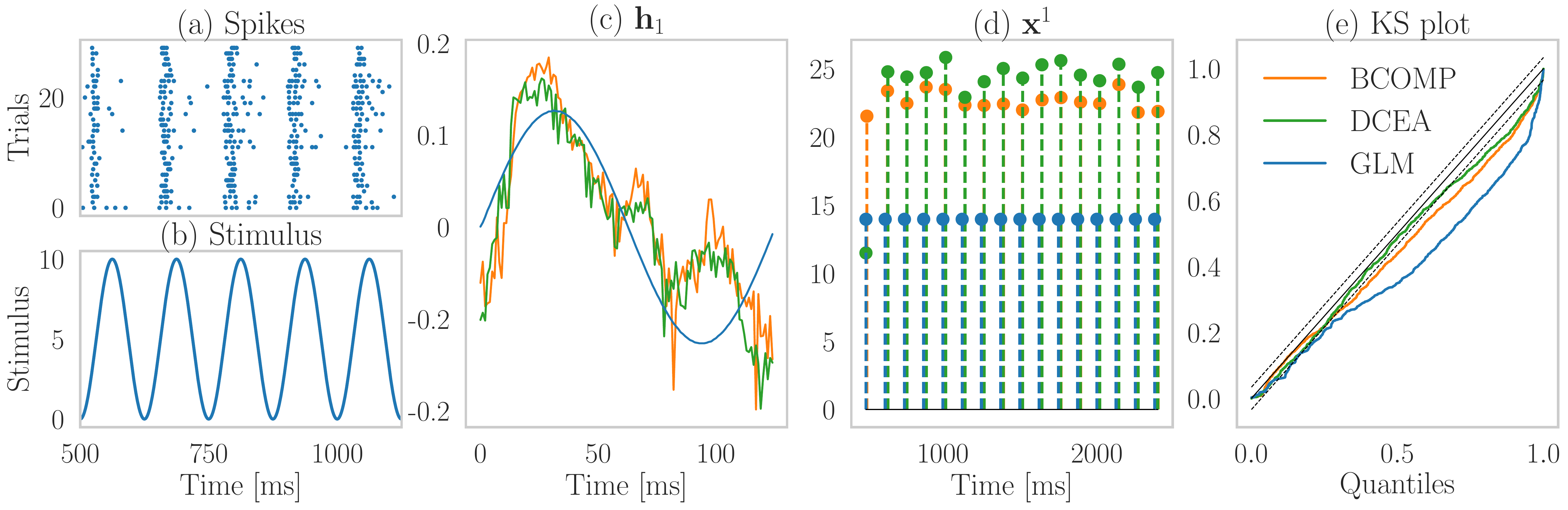}
	\end{minipage}
	\caption{A segment of data from a neuron and result of applying DCEA and BCOMP. (a) A dot indicates a spike from the neuron. (b) Stimulus used to move the whisker. (c) Whisker velocity covariate (blue) used in GLM analysis, along with whisker velocities estimated with BCOMP (orange) and DCEA (green) using all 10 neurons in the dataset. The units are $\frac{\text{mm}}{10}$ per ms. (d) The estimated sparse codes (onset of whisker deflection). (e) Analysis of Goodness-of-fit  using KS plots. The dotted lines represent 95\% confidence intervals. }
	\label{fig:whisker}
\end{figure*}

\subsection{Neural spiking data from somatosensory thalamus}
We now apply DCEA to neural spiking data from somatosensory thalamus of rats recorded in response to periodic whisker deflections~\cite{Temereanca2008}. The objective is to learn the features of whisker motion that modulate neural spiking strongly. In the experiment, a piezoelectric simulator controls whisker position using an \emph{ideal} position waveform. 
%
As the interpretability of the learned filters is important, we constrain the weights of encoder and decoder to be $\Bigh$. DECA lets us learn, in an \textit{unsupervised} fashion, the features that best explains the data.

The dataset consists of neural spiking activity from $J=10$ neurons in response to periodic whisker deflections. Each example $j$ consists of $M_j=50$ trials lasting $3{,}000$ ms, i.e., $\y^{j,m}\in\mathbb{R}^{3000}$. Fig.~\ref{fig:whisker}(a) depicts a segment of data from a neuron. Each trial begins/ends with a baseline period of $500$ ms. During the middle $2{,}000$ ms, a periodic deflection with period $125$ ms is applied to a whisker by the piezoelectric stimulator. There are $16$ total deflections, five of which are shown in Fig.~\ref{fig:whisker}(b). The stimulus represents ideal whisker position. The blue curve in Fig.~\ref{fig:whisker}(c) depicts the whisker velocity obtained as the first derivative of the stimulus.

\paragraph{Methods}
We compare DCEA to $\ell_0$-based ECDL using BCOMP (introduced in the previous section), and a generalized linear model (GLM)~\cite{glm} with whisker-velocity covariate~\cite{whisker2014}. For all three methods, we let $C=1$ and $\smlh_1\in \mathbb{R}^{125}$, initialized using the whisker velocity (Fig.~\ref{fig:whisker}(c), blue). We set $\lambda=0.119$ for DCEA and set the sparsity level of BCOMP to $16$. As in the simulation, we used $\prox_{b}=\text{ReLU}_{b}$ to ensure non-negativity of the codes. We used $30$ trials from each neuron to learn $\smlh_1$ and the remaining $20$ trials as a test set to assess goodness-of-fit. We describe additional parameters used for DCEA and the post-processing steps in the {\bf Appendix}.

\paragraph{Results}
The orange and green curves from~Fig.~\ref{fig:whisker}(c) depict the estimates of whisker velocity computed from the neural spiking data using BCOMP and DCEA, respectively. The figure indicates that the spiking activity of this population of 10 neurons encodes well the whisker velocity, and is most strongly modulated by the maximum velocity of whisker movement.

Fig.~\ref{fig:whisker}(d) depicts the $16$ sparse codes that accurately capture the onset of stimulus in each of the $16$ deflection periods. The heterogeneity of amplitudes estimated by DCEA and BCOMP is indicative of the variability of the neural response to whisker deflections repeated 16 times, possibly capturing other characteristics of cellular and circuit response dynamics (e.g., adaptation). 
This is in sharp contrast to the GLM--detailed in the {\bf Appendix}--which uses the ideal whisker velocity (Fig.~\ref{fig:whisker}(c), blue) as a covariate, and assumes that neural response to whisker deflections is constant across deflections.

In Fig.~\ref{fig:whisker}(e), we use the Kolmogorov-Smirnov (KS) test to compare how well DCEA, BCOMP, and the GLM fit the data for a representative neuron in the dataset~\cite{Brown2002}. KS plots are a visualization of the KS test for assessing the Goodness-of-fit of models to point-process data, such as neural spiking data (see {\bf Appendix} for details). The figure shows that DCEA and BCOMP are a much better fit to the data than the GLM. 

We emphasize that 1) the similarity of the learned $\mathbf{h}_1$ and 2) the similar goodness-of-fit of DCEA and BCOMP to the data shows that DCEA performs ECDL. In addition, this analysis shows the power of the ECDL as an \textit{unsupervised} and \textit{data-driven} approach for data analysis, and a superior alternative to GLMs, where the features are hand-crafted.
\section{Conclusion}
We introduced a class of neural networks based on a generative model for convolutional dictionary learning (CDL) using data from the natural exponential-family, such as count-valued and binary data. The proposed class of networks, which we termed deep convolutional exponential auto-encoder (DCEA), is competitive compared to state-of-the-art supervised Poisson image denoising algorithms, with an order of magnitude fewer trainable parameters.

We analyzed gradient dynamics of shallow exponential-family auto-encoder (i.e., unfold the encoder once) for binomial distribution and proved that when trained with approximate gradient descent, the network recovers the dictionary corresponding to the binomial generative model.

We also showed using binomial data simulated according to the convolutional exponential-family generative model that DCEA performs dictionary learning, in an unsupervised fashion, when the parameters of the encoder/decoder are constrained. The application of DCEA to neural spike data suggests that DCEA is superior to GLM analysis, which relies on hand-crafted covariates.

\section*{Acknowledgements}

The authors gratefully acknowledge supports by NSF-Simons Center for Mathematical and Statistical Analysis of Biology at Harvard University (supported by NSF grant no. DMS-1764269), the Harvard FAS Quantitative Biology Initiative, and the Samsung scholarship. This research is also supported by AWS Machine Learning Research Awards. The authors also thank the reviewers for their insightful comments.

\newpage
\bibliography{icml2020_arxiv_final}
\bibliographystyle{icml2020}

\newpage
\onecolumn
\begin{center}
	\hrule height 1pt
	\vspace{14pt}
	\textbf{\large Appendix for convolutional dictionary learning based auto-encoders\\
	for natural exponential-family distributions}
	\vspace{12pt}
	\hrule height 1pt
	\vspace{12pt}
\end{center}
\date{}
\thispagestyle{empty}
\vspace{-2mm}
\section{Gradient dynamics of shallow exponential auto-encoder (SEA)}
\begin{theorem} \label{theo:descentinformal}
(informal). Given a ``good'' initial estimate of the dictionary from the binomial dictionary learning problem, and infinitely many examples, the binomial SEA, when trained by gradient descent through backpropagation, learns the dictionary. 
\end{theorem}
\begin{theorem}\label{theo:descentrelu}
Suppose the generative model satisfies (A1) - (A14). Given infinitely many examples (i.e., $J\rightarrow \infty$), the binomial SEA with $\prox_{\bvec} = \text{ReLU}_{\bvec}$ trained by approximate gradient descent followed by normalization using the learning rate of $\kappa = O(p/s)$ (i.e., $\smlw^{(l+1)}_i = \text{normalize}(\smlw^{(l)}_i - \kappa g_i)$) recovers $\A$. More formally, there exists $\delta$$\in$$(0,1)$ such that at every iteration $l$, $\forall i\ \| \smlw_i^{(l+1)} - \smla_i \|_2^2 \leq (1- \delta) \| \smlw_i^{(l)} - \smla_i \|_2^2 + \kappa \cdot O(\frac{\max(s^2, s^3/ p^{\frac{2}{3} + 2\xi})}{p^{1+6\xi}})$.
\end{theorem}
\begin{theorem}\label{theo:descentht}
Suppose the generative model satisfies (A1) - (A13). Given infinitely many examples (i.e., $J\rightarrow \infty$), the binomial SEA with $\prox_{\bvec} = \text{HT}_{\bvec}$ trained by approximate gradient descent followed by normalization using the learning rate of $\kappa = O(p/s)$ (i.e., $\smlw^{(l+1)}_i = \text{normalize}(\smlw^{(l)}_i - \kappa g_i)$) recovers $\A$. More formally, there exists $\delta$$\in$$(0,1)$ such that at every iteration $l$, $\forall i\ \| \smlw_i^{(l+1)} - \smla_i \|_2^2 \leq (1- \delta) \| \smlw_i^{(l)} - \smla_i \|_2^2 + \kappa \cdot O(\frac{\max(s^2, s^3/ p^{\frac{2}{3} + 2\xi})}{p^{1+6\xi}})$.
\end{theorem}
\noindent In proof of the above theorem, our approach is similar to~\cite{Nguyen2019}.
\subsection{Generative model and architecture}
We have $J$ binomial observations $\y^j = \sum_{m=1}^{M_j} \one^j_m$ where $\y^j$ can be seen as sum of $M_j$ independent Bernoulli random variables (i.e., $\one^j_m$). We can express $\sigma^{-1}(\boldmu) = \A \x^*$, where $\sigma(z) = \frac{\e^{z}}{1 + \e^{z}}$ is the inverse of the corresponding link function (sigmoid), $\A \in \R^{n \times p}$ is a matrix dictionary, and $\x^* \in \R^{p}$ is a sparse vector. Hence, we have
\begin{equation}\label{eq:sigmoid}
E[\y^j] = \boldmu = \frac{\e^{\A\x^*}}{1 + \e^{\A\x^*}} = \sigma(\A\x^*).
\end{equation}
\noindent In this analysis, we assume that there are infinitely many examples (i.e., $J \rightarrow \infty$), hence, we use the expectation of the gradient for backpropagation at every iteration. We also assume that there are infinite number of Bernoulli observation for each binomial observation (i.e., $M_j \rightarrow \infty$). Hence, from the Law of Large Numbers, we have the following convergence in probability
\begin{equation}\label{eq:mean}
\lim_{M_j \rightarrow \infty} \frac{1}{M_j} \y^j = \lim_{M_j \rightarrow \infty} \frac{1}{M_j} \sum_{m=1}^{M_j} \one_m^j= \boldmu = \sigma(\A\x^*),
\end{equation}
\noindent We drop $j$ for ease of notation. Algorithm~\ref{algo:sea} shows the architecture when the code is initialized to $\mathbf{0}$. $\W \in \R^{n \times p}$ are the weights of the auto-encoder. The encoder is unfolded only once and the step size of the proximal mapping is set to $4$ (i.e., assuming the maximum singular value of $\A$ is $1$, then $4$ is the largest step size to ensure convergence of the encoder as the first derivative of sigmoid is bounded by $\frac{1}{4}$.
\begin{algorithm}
\caption{SEA.}
\label{algo:sea}
\begin{algorithmic}
\INPUT $\y, \W, \bvec$
\OUTPUT $ \cvec_2$\\
$\cvec_1 = 4 \W^{\text{T}} (\y - \frac{1}{2})$\\
$\x  =  \prox_{\bvec}(\cvec_1)$\\
$\cvec_2  =  \W \x$
\end{algorithmic}
\end{algorithm}

\noindent For Theorem~\ref{theo:descentrelu}, $\prox_{\bvec}(\z)$ is an element-wise operator where $\prox_{b_i}(z_i) = \text{ReLU}_{b_i}(z_i) = z_i\cdot \one_{|z_i|\geq b_i}$, and $\frac{1}{2}= \sigma(\mathbf{0})$ appears in the first layer as of the initial code estimate is $\mathbf{0}$. From the definition of ReLU, we can see that $\x  =  \prox_{\bvec}(\cvec_1) =  \one_{\x \neq 0} (\cvec_1 - \bvec)$ where $\one_{\x \neq 0}$ is an indicator function. For Theorem~\ref{theo:descentht}, $\prox_{\bvec}(\z)$ is an element-wise operator where $\prox_{b_i}(z_i) = \text{HT}_{b_i}(z_i) = z_i\cdot \one_{|z_i|\geq b_i}$, and $\x  =  \prox_{\bvec}(\cvec_1) =  \one_{\x \neq 0} \cvec_1$.

\subsection{Assumptions and definitions}
\noindent Given the following definition and notations,
\begin{itemize}
\item [(D1)] $\W$ is $q$-close to $\A$ if there is a permutation $\pi$ and sign flip operator $u$ such that $\forall i\ \| u(i) \smlw_{\pi(i)} - \smla_i \|_2 \leq q$.
\item [(D2)] $\W$ is $(q, \varepsilon)$-near to $\A$ if $\W$ is $q$-close to $\A$ and $\|\W - \A \|_2 \leq \varepsilon \|\A\|_2$.
\item [(D3)] A unit-norm columns matrix $\A$ is $\eta$-incoherent if for every pair $(i,j)$ of columns, $| \langle \smla_i, \smla_j \rangle | \leq \frac{\eta}{\sqrt{n}}$.
\item [(D4)] We define column $i$ of $\W$ as $\smlw_i$.
\item [(D5)] $\smlw_i$ is $\tau_i$-correlated to $\smla_i$ if $\tau_i = \langle \smlw_i, \smla_i \rangle = \smlw_i^{\text{T}} \smla_i$. Hence, $\| \smlw_i - \smla_i \|_2^2 = 2 (1 - \tau_i)$.
\item [(D6)] From the binomial likelihood, the loss would be $\lim_{M \rightarrow \infty}  \mathcal{L}_{\W}({\y, \W\x)} = \lim_{M \rightarrow \infty}  - \frac{1}{M} \big(\W\x\big)^{\text{T}} \y + \mathbf{1}_n^{\text{T}} \log \big(1+\exp\big(\W\x\big)\big)$.
\item [(D7)] We denote the expectation of the gradient of the loss defined in (D6) with respect to $\smlw_i$ to be $g_i = E[ \lim_{M \rightarrow \infty} \frac{\partial \mathcal{L}_{\W}}{\partial \smlw_i}]$.
\item [(D8)] $\W_{\backslash i}$ denotes the matrix $\W$ with column $i$ removed, and $S^{\backslash i}$ denotes $S$ excluding $i$.
\item [(D9)] $[\mathbf{z}]_d$ denotes $z_d$ (i.e., the $d^{\text{th}}$ element fo the vector $\mathbf{z}$).
\item [(D10)] $[p]$ denotes the set $\{1, \ldots, p\}$, and $[p]^{\backslash i}$ denotes $[p]$ excluding $i$.
\item [(D11)] For $\A \in \R^{n \times p}$, $\A_S \in \R^{n \times s}$ indicates a matrix with columns from the set $S$. Similarly, for $\x^* \in \R^{p}$, $\x^*_S \in \R^{s}$ indicates a vector containing only the elements with indices from $S$.
\end{itemize}
\noindent we assume the generative model satisfies the following assumptions:
\begin{itemize}
\item [(A1)] Let the code $\x^*$ be $s$-sparse and have support S (i.e., $\text{supp}(\x) = S$) where each element of $S$ is chosen uniformly at random without replacement from the set $[p]$. Hence, $p_i = P(i \in S) = s/p$ and $p_{ij} = P(i,j \in S) = s(s-1)/(p (p-1))$.
\item [(A2)] Each code is bounded (i.e., $| x_i | \in [L_x, C_x]$) where $0 \leq L_x \leq C_x$ and $C_x = O(\frac{1}{p^{\frac{1}{3}+\xi}})$, and $\xi >0$. Then $\|\x_S^* \|_2 \leq \sqrt{s} C_x$. For the case when $\prox_{\bvec} = \text{ReLU}_{\bvec}$, we assume the code is non-negative.
\item [(A3)] Given the support, we assume $\x^*_S$ is i.i.d, zero-mean, and has symmetric probability density function. Hence, $E[\x_i^* \mid S] = 0$ and $E[\x_S^* \x_S^{\text{T}} \mid S] = \nu \eye$ where $\nu \leq C_x$.
\item [(A4)] From the forward pass of the encoder, $\text{supp}(\x) = \text{supp}(\x^*) = S$ with high probability. We call this code consistency, a similar definition from \cite{Nguyen2019}. This code consistency enforces some conditions (i.e., based on $L_x$ and $C_x$ for ReLU and $L_x$ for HT) on the value of $\bvec$ which we do not explicitly express. For when $\prox_{\bvec} = \text{ReLU}_{\bvec}$, $\W \x = \W_S \x_S = 4 \W_S \W_S^{\text{T}} (\y- \frac{1}{2}) - \W_S \bvec_S$, and for when $\prox_{\bvec} = \text{HT}_{\bvec}$, $\W \x = 4 \W_S \W_S^{\text{T}} (\y- \frac{1}{2})$.
\item [(A5)] We assume $\forall i\ \| \smla_i \|_2 = 1$.
\item [(A6)] Given $s < n \leq p$, we have $\| \A \|_2 = O(\sqrt{p/n})$ and $\| \A_S \|_2 = O(1)$.
\item [(A7)] $\W$ is $(q, 2)$-near $\A$; thus, $\| \W \|_2 \leq \| \W - \A \|_2 + \| \A \|_2 \leq O(\sqrt{p/n})$.
\item [(A8)] $\A$ is $\eta$-incoherent.
\item [(A9)] $\smlw_i$ is $\tau_i$-correlated to $\smla_i$.
\item [(A10)] For any $i \neq j$, we have $| \langle \smlw_i, \smla_j \rangle | = | \langle \smla_i, \smla_j \rangle + \langle \smlw_i - \smla_i, \smla_j \rangle | \leq \frac{\eta}{\sqrt{n}} + \| \smlw_i - \smla_i \|_2 \| \smla_j\|_2 \leq \frac{\eta}{\sqrt{n}} + q$.
\item [(A11)] We assume the network is trained by approximate gradient descent followed by normalization using the learning rate of $\kappa$. Hence, the gradient update for column $i$ at iteration $l$ is $\smlw^{(l+1)}_i = \smlw^{(l)}_i - \kappa g_i$. At the normalization step, $\forall i$, we enforce $\| \smlw_i \|_2 = 1$. Lemma 5 in~\cite{Nguyen2019} shows that descent property can also be achieved with the normalization step.
\item [(A12)] We use the Taylor series of $\sigma(z)$ around $0$. Hence, $\sigma(z) = \frac{1}{2} + \frac{1}{4} z + \nabla^2\sigma(\bar z) (z)^2$, where $0 \leq \bar z \leq z$ and $\nabla^2$ denotes Hessian.
\item [(A13)] To simplify notation, we assume that the permutation operator $\pi(.)$ is identity and the sign flip operator $u(.)$ is $+1$.
\item [(A14)] When $\prox_{\bvec} = \text{ReLU}_{\bvec}$, at every iteration of the gradient descent, given $\tau_i$, the bias $\bvec$ in the network satisfies $|\nu \tau_i (\tau_i - 1) + b_i^2| \leq 2 \tau_i (1 - \tau_i)$.
\end{itemize}
\subsection{Non-negative sparse coding with $\prox_{\bvec} = \text{ReLU}_{\bvec}$}
\subsubsection{Gradient derivation}
First, we derive $g_i$ when $\prox_{\bvec} = \text{ReLU}_{\bvec}$. In this derivation, by dominated convergence theorem, we interchange the limit and derivative. We also compute the limit inside $\sigma(.)$ as it is a continuous function.
\begin{equation}
\begin{aligned}
&\lim_{M \rightarrow \infty} \frac{\partial \mathcal{L}_{\W}}{\partial \smlw_i} = \lim_{M \rightarrow \infty} \frac{\partial \cvec_1}{\partial \smlw_i}  \frac{\partial \mathcal{L}_{\W}}{\partial \cvec_1} + \frac{\partial \cvec_{2}}{\partial \smlw_i}  \frac{\partial \mathcal{L}_{\W}}{\partial \cvec_2} = \frac{\partial \cvec_1}{\partial \smlw_i} \frac{\partial \x}{\partial \smlc_1} \frac{\partial \cvec_2}{\partial \x}\ \frac{\mathcal{L}_{\W}}{\partial \cvec_2}+ \frac{\partial \cvec_{2}}{\partial \smlw_i} \frac{\partial \mathcal{L}_{\W}}{\partial \cvec_2}\\
&= \left(\underbrace{[0, 0, \ldots, 4(\boldmu - \frac{1}{2}), \ldots, 0]}_{n \times p} \underbrace{\text{diag}(\prox_{\bvec}^{\prime}(\cvec_{1}))}_{p \times p} \W^{\text{T}} + \one_{\x_i \neq 0} (\smlw_i^{\text{T}} 4(\boldmu - \frac{1}{2}) \eye - b_i \eye) \right)\\
&\times \left(- \sigma(\A\x^*) +  \sigma(4\W_S\W_S^{\text{T}} (\boldmu - \frac{1}{2}) - \W_S \bvec_S)\right)\\
&=  \left( \prox_{b_i}^{\prime}(\cvec_{1,i}) 4(\boldmu - \frac{1}{2}) \smlw_i^{\text{T}} + \one_{\x_i \neq 0} (\smlw_i^{\text{T}} 4(\boldmu - \frac{1}{2}) - b_i) \eye \right) \left(\sigma(4\W_S\W_S^{\text{T}} (\boldmu -\frac{1}{2}) - \W_S \bvec_S) - \sigma(\A\x^*)\right).
\end{aligned}
\end{equation}
\noindent We further expand the gradient, by replacing $\sigma(.)$ with its Taylor expansion. We have
\begin{equation}
\sigma(\A\x^*) = \frac{1}{2} + \frac{1}{4} \A\x^* + \boldeps,
\end{equation}
\noindent where $\boldeps = [\epsilon_1, \ldots, \epsilon_n]^{\text{T}}$, $\epsilon_d = \nabla^2\sigma(u_d)([\A\x^*]_d)^2$, and  $0 \leq u_d \leq [\A\x^*]_d$. Similarly,
\begin{equation}
\sigma(4\W_S\W_S^{\text{T}} (\boldmu - \frac{1}{2}) - \W_S \bvec_S) = \frac{1}{2} + \W_S\W_S^{\text{T}} (\boldmu - \frac{1}{2}) - \frac{1}{4} \W_S \bvec_S + \tilde \boldeps,
\end{equation}
\noindent where $\tilde \boldeps = [\tilde \epsilon_1, \ldots, \tilde \epsilon_n]^{\text{T}}$, $\tilde \epsilon_d = \nabla^2\sigma(\tilde u_d) ([4 \W_S\W_S^{\text{T}} (\boldmu - \frac{1}{2}) - \W_S \bvec_S]_d)^2$ , and $0 \leq \tilde u_d \leq [4 \W_S\W_S^{\text{T}} (\boldmu - \frac{1}{2}) - \W_S \bvec_S]_d$. Again, replacing $\boldmu$ with Taylor expansion of $\sigma(\A\x^*)$, we get
\begin{equation}
\sigma(4 \W_S\W_S^{\text{T}} (\boldmu - \frac{1}{2}) - \W_S \bvec_S) = \frac{1}{2} + \W_S\W_S^{\text{T}} (\frac{1}{4} \A\x^* + \boldeps) - \frac{1}{4} \W_S \bvec_S + \tilde \boldeps.
\end{equation}
\noindent By symmetry, $E[\boldeps \mid S] = E[\tilde \boldeps \mid S] = 0$. The expectation of gradient $g_i$ would be
\begin{equation}
g_i = E[\one_{\x_i \neq 0} \left( ( \A\x^* + 4 \boldeps) \smlw_i^{\text{T}} + (\A\x^* + 4 \boldeps) \eye - b_i \eye \right) (\frac{1}{4} ( \W_S\W_S^{\text{T}} - \eye) (\A\x^*) - \frac{1}{4} \W_S \bvec_S + (\W_S\W_S^{\text{T}} - \eye) \boldeps + \tilde \boldeps)].
\end{equation}
\subsubsection{Gradient dynamics}
\noindent Given the code consistency from the forward pass of the encoder, we replace $\one_{\x_i \neq 0}$ with $\one_{\x_i^* \neq 0}$ and denote the error by $\gamma$ as below which is small for large $p$~\cite{Nguyen2019}.
\begin{equation}
\gamma = E[(\one_{\x_i^* \neq 0} - \one_{\x_i \neq 0}) \left( ( \A\x^* + 4 \boldeps) \smlw_i^{\text{T}} +\smlw_i^{\text{T}} (\A\x^* + 4 \boldeps) \eye - b_i \eye \right) (\frac{1}{4} ( \W_S\W_S^{\text{T}} - \eye) (\A\x^*) - \frac{1}{4} \W_S \bvec_S + (\W_S\W_S^{\text{T}} - \eye) \boldeps + \tilde \boldeps)].
\end{equation}
\noindent Now, we write $g_i$ as
\begin{equation}
g_i = E[\one_{\x_i^* \neq 0} \left( ( \A\x^* + 4 \boldeps) \smlw_i^{\text{T}} + \smlw_i^{\text{T}} (\A\x^* + 4 \boldeps) \eye - b_i \eye \right) (\frac{1}{4} ( \W_S\W_S^{\text{T}} - \eye) (\A\x^*) - \frac{1}{4} \W_S \bvec_S  + (\W_S\W_S^{\text{T}} - \eye) \boldeps + \tilde \boldeps)] + \gamma.
\end{equation}
\noindent We can see that if $i \notin S$ then $\one_{\x_i^*} = 0$ hence, $g_i = 0$. Thus, in our analysis, we only consider the case $i \in S$. We decompose $g_i$ as below.
\begin{equation}
g_i = g_i^{(1)} + g_i^{(2)} + g_i^{(3)} + \gamma,
\end{equation}
\noindent where
\begin{equation}
\begin{aligned}
g_i^{(1)} &= E[\frac{1}{4} \smlw_i^{\text{T}} (\A\x^* + 4 \boldeps) (\W_S\W_S^{\text{T}} - \eye) \A\x^*].
\end{aligned}
\end{equation}
\begin{equation}
\begin{aligned}
g_i^{(2)} &= E[\frac{1}{4} (\A\x^* + 4 \boldeps) \smlw_i^{\text{T}} (\W_S\W_S^{\text{T}} - \eye) \A\x^*] .
\end{aligned}
\end{equation}
\begin{equation}
\begin{aligned}
g_i^{(3)} &= E[\left( ( \A\x^* + 4 \boldeps) \smlw_i^{\text{T}} + \smlw_i^{\text{T}} (\A\x^* + 4 \boldeps) \eye \right) ((\W_S\W_S^{\text{T}} - \eye) \boldeps + \tilde \boldeps)].
\end{aligned}
\end{equation}
\begin{equation}
\begin{aligned}
g_i^{(4)} &= E[(-b_i) ((\W_S\W_S^{\text{T}} - \eye) \boldeps + \tilde \boldeps)].
\end{aligned}
\end{equation}
\begin{equation}
\begin{aligned}
g_i^{(5)} &= E[(-b_i) (\frac{1}{4} ( \W_S\W_S^{\text{T}} - \eye) (\A\x^*) - \frac{1}{4} \W_S \bvec_S)].
\end{aligned}
\end{equation}
\begin{equation}
\begin{aligned}
g_i^{(6)} &= E[ \left( ( \A\x^* + 4 \boldeps) \smlw_i^{\text{T}} + \smlw_i^{\text{T}} (\A\x^* + 4 \boldeps) \eye \right) (- \frac{1}{4} \W_S \bvec_S)].
\end{aligned}
\end{equation}
\noindent We define
\begin{equation}
\begin{aligned}
g_{i,S}^{(1)} &= E[\frac{1}{4} \smlw_i^{\text{T}} (\A\x^* + 4 \boldeps) (\W_S\W_S^{\text{T}} - \eye) \A\x^* \mid S].
\end{aligned}
\end{equation}
\begin{equation}
\begin{aligned}
g_{i,S}^{(2)} &= E[\frac{1}{4} (\A\x^* + 4 \boldeps) \smlw_i^{\text{T}} (\W_S\W_S^{\text{T}} - \eye) \A\x^* \mid S].
\end{aligned}
\end{equation}
\begin{equation}
\begin{aligned}
g_{i,S}^{(3)} &= E[\left( ( \A\x^* + 4 \boldeps) \smlw_i^{\text{T}} + \smlw_i^{\text{T}} (\A\x^* + 4 \boldeps) \eye \right) ((\W_S\W_S^{\text{T}} - \eye) \boldeps + \tilde \boldeps) \mid S].
\end{aligned}
\end{equation}
\begin{equation}
\begin{aligned}
g_{i,S}^{(4)} &= E[(-b_i) ((\W_S\W_S^{\text{T}} - \eye) \boldeps + \tilde \boldeps) \mid S].
\end{aligned}
\end{equation}
\begin{equation}
\begin{aligned}
g_{i,S}^{(5)} &= E[(-b_i) (\frac{1}{4} ( \W_S\W_S^{\text{T}} - \eye) (\A\x^*) - \frac{1}{4} \W_S \bvec_S) \mid S].
\end{aligned}
\end{equation}
\begin{equation}
\begin{aligned}
g_{i,S}^{(6)} &= E[ \left( ( \A\x^* + 4 \boldeps) \smlw_i^{\text{T}} + \smlw_i^{\text{T}} (\A\x^* + 4 \boldeps) \eye \right) (- \frac{1}{4} \W_S \bvec_S) \mid S].
\end{aligned}
\end{equation}
\noindent Hence, $g_{i}^{(k)} = E[g_{i,S}^{(k)}]$ for $k=1, \ldots , 6$, where the expectations are with respect to the support S.
\begin{equation}
\begin{aligned}
g_{i,S}^{(1)} &= E[\frac{1}{4} \smlw_i^{\text{T}} (\A\x^* + 4 \boldeps) (\W_S\W_S^{\text{T}} - \eye) \A\x^* \mid S]\\
&= \sum_{j,l \in S}  E[\frac{1}{4} \smlw_i^{\text{T}} \smla_j \x_j^* (\W_S\W_S^{\text{T}} - \eye) \smla_l\x_l^* \mid S] + E[\smlw_i^{\text{T}} \boldeps (\W_S\W_S^{\text{T}} - \eye) \A\x^* \mid S]\\
&= \frac{1}{4} \nu \smlw_i^{\text{T}} \smla_i (\W_S\W_S^{\text{T}} - \eye) \smla_i + \sum_{l \in S^{\backslash i}} \frac{1}{4} \nu \smlw_i^{\text{T}} \smla_l (\W_S\W_S^{\text{T}} - \eye) \smla_l + e_1
= \frac{1}{4} \nu \smlw_i^{\text{T}} \smla_i (\W_S\W_S^{\text{T}} - \eye) \smla_i + r_1 + e_1.
\end{aligned}
\end{equation}
\noindent We denote $r_1 = \sum_{l \in S^{\backslash i}} \frac{1}{4} \nu \smlw_i^{\text{T}} \smla_l (\W_S\W_S^{\text{T}} - \eye) \smla_l$ and $e_1 = E[\smlw_i^{\text{T}} \boldeps (\W_S\W_S^{\text{T}} - \eye) \A\x^* \mid S]$. Similarly for $g_{i,S}^{(2)}$, we have
\begin{equation}
g_{i,S}^{(2)} = \frac{1}{4} \nu \smla_i \smlw_i^{\text{T}} (\W_S\W_S^{\text{T}} - \eye) \smla_i  + r_2 + e_2.
\end{equation}
\noindent We denote $r_2 = \sum_{l \in S^{\backslash i}} \frac{1}{4} \nu \smla_l \smlw_i^{\text{T}} (\W_S\W_S^{\text{T}} - \eye) \smla_l$, $e_2 = E[\boldeps \smlw_i^{\text{T}} (\W_S\W_S^{\text{T}} - \eye) \A\x^* \mid S]$, $e_3 = g_{i,S}^{(3)}$, and $e_4 = g_{i,S}^{(4)}$. We compute $g_{i,S}^{(5)}$ and $g_{i,S}^{(6)}$ next.
\begin{equation}
\begin{aligned}
g_{i,S}^{(5)} &= E[(-b_i) (\frac{1}{4} ( \W_S\W_S^{\text{T}} - \eye) (\A\x^*) - \frac{1}{4} \W_S \bvec_S) \mid S]
= \frac{1}{4} b_i^2 \smlw_i + \frac{1}{4} b_i \sum_{j \in S^{\backslash i}} \smlw_j b_j\\
\end{aligned}
\end{equation}
\begin{equation}
\begin{aligned}
g_{i,S}^{(6)} &= E[ \left( ( \A\x^* + 4 \boldeps) \smlw_i^{\text{T}} + \smlw_i^{\text{T}} (\A\x^* + 4 \boldeps) \eye \right) (- \frac{1}{4} \W_S \bvec_S) \mid S] = 0
\end{aligned}
\end{equation}
\noindent We denote $\beta = E[r_1 + r_2 + e_1 + e_2 + e_3 + e_4] + \gamma$. Combining the terms,
\begin{equation}
\begin{aligned}
g_i &= E[\frac{1}{4} \nu \smlw_i^{\text{T}} \smla_i (\W_S\W_S^{\text{T}} - \eye) \smla_i + \frac{1}{4} \nu \smla_i \smlw_i^{\text{T}}(\W_S\W_S^{\text{T}} - \eye) \smla_i + \frac{1}{4} b_i^2 \smlw_i + \frac{1}{4} b_i \sum_{j \in S^{\backslash i}} \smlw_j b_j] + \beta \\
&= E[- \frac{1}{2} \nu \tau_i \smla_i + \frac{1}{4} \nu \tau_i \sum_{j \in S} \smlw_j \smlw_j^{\text{T}} \smla_i + \frac{1}{4} \nu \smla_i \smlw_i^{\text{T}} \sum_{j \in S} \smlw_j \smlw_j^{\text{T}} \smla_i + \frac{1}{4} b_i^2 \smlw_i + \frac{1}{4} b_i \sum_{j \in S^{\backslash i}} \smlw_j b_j] + \beta \\
&= E[- \frac{1}{2} \nu \tau_i \smla_i + \frac{1}{4} \nu \tau_i^2 \smlw_i + \frac{1}{4} \nu \tau_i \sum_{j \in S^{\backslash i}} \smlw_j \smlw_j^{\text{T}} \smla_i +  \frac{1}{4} \nu \tau_i  \| \smlw_i\|_2^2 \smla_i\\
&+ \frac{1}{4} \nu \left(\smla_i \smlw_i^{\text{T}} \right) \sum_{j \in S^{\backslash i}} \smlw_j \smlw_j^{\text{T}} \smla_i + \frac{1}{4} b_i^2 \smlw_i + \frac{1}{4} b_i \sum_{j \in S^{\backslash i}} \smlw_j b_j] + \beta \\
&= - \frac{1}{4} p_i \nu \tau_i \smla_i + p_i \frac{1}{4} (\nu \tau_i^2 + b_i^2) \smlw_i + \zeta + \beta,
\end{aligned}
\end{equation}
\noindent where $\zeta = \sum_{j \in [p]^{\backslash i}} \frac{1}{4} p_{ij} \nu \tau_i \smlw_j \smlw_j^{\text{T}} \smla_i + \frac{1}{4} p_{ij} \nu \smla_i \smlw_i^{\text{T}} \smlw_j \smlw_j^{\text{T}} \smla_i + \frac{1}{4} p_{ij} b_i b_j \smlw_j$. We continue
\begin{equation}
g_i = \frac{1}{4} p_i \nu \tau_i (\smlw_i - \smla_i)  + v,
\end{equation}
\noindent where we denote $v =\frac{1}{4} p_i (\nu \tau_i (\tau_i - 1) + b_i^2) \smlw_i + \zeta + \beta$.
\begin{lemma}\label{lemma:v_relu}
Suppose the generative model satisfies $(A1) - (A14)$. Then
\begin{equation}
\| v \|_2 \leq \frac{1}{4} p_i \tau_i \nu q \| \smlw_i - \smla_i \|_2  + O(\max(C_x^3 s \sqrt{s}, C_x^4 s^2))
\end{equation}
\end{lemma}
\begin{proof}
\begin{equation}
\begin{aligned}
\| \zeta \|_2  &=  \| \sum_{j \in [p]^{\backslash i}} \frac{1}{4} p_{ij} \nu \tau_i \smlw_j \smlw_j^{\text{T}} \smla_i + \frac{1}{4} p_{ij} \nu \smla_i \smlw_i^{\text{T}} \smlw_j \smlw_j^{\text{T}} \smla_i + \frac{1}{4} p_{ij} b_i b_j \smlw_j\|_2 \\
&= \| \frac{1}{4} p_{ij} \nu \tau_i  \W_{\backslash i} \W_{\backslash i}^{\text{T}} \smla_i + \frac{1}{4} p_{ij} \nu \smla_i \smlw_i^{\text{T}} \W_{\backslash i} \W_{\backslash i}^{\text{T}} \smla_i + \frac{1}{4} p_{ij} b_i \W_{\backslash i} b_{\backslash i}\|_2 \\
&\leq \frac{1}{4} p_{ij} \nu \tau_i \| \W_{\backslash i} \|_2^2 \| \smla_i\|_2 + \frac{1}{4} p_{ij} \nu \| \smlw_i \|_2 \| \W_{\backslash i}\|_2^2 \| \smla_i\|_2^2 + \frac{1}{4} p_{ij} | b_i | \| \W_{\backslash i} \|_2 \| b_{\backslash i} \|_2  \\
&= \frac{1}{4} O(\nu \tau_i s^2 / (np)) + \frac{1}{4} O(\nu s^2 / (np)) + \frac{1}{4} O(\sqrt{\frac{p}{n}} s^2/p^2) = O(s^2/(np)).
\end{aligned}
\end{equation}
\begin{equation}
\begin{aligned}
&\| E[ r_1] \|_2 = \| E[\sum_{l \in S^{\backslash i}} \frac{1}{4} \nu \smlw_i^{\text{T}} \smla_l ( \W_S\W_S^{\text{T}} - \eye) \smla_l] \|_2
= \| E[\sum_{l \in S^{\backslash i}} \frac{1}{4} \nu \smlw_i^{\text{T}} \smla_l  \W_S\W_S^{\text{T}} \smla_l - \sum_{l \in S^{\backslash i}} \frac{1}{4} \nu \smlw_i^{\text{T}} \smla_l \smla_l] \|_2 \\
=& \| \sum_{l \neq j \neq i} p_{ijl} \frac{1}{4} \nu \smlw_i^{\text{T}} \smla_l \smlw_j \smlw_j^{\text{T}} \smla_l +  \sum_{j \neq i} p_{ij} \frac{1}{4} \nu \smlw_i^{\text{T}} \smla_j \smlw_j \smlw_j^{\text{T}}\smla_j + p_{il} \frac{1}{4} \nu \smlw_i^{\text{T}} \smla_l \smlw_i \smlw_i^{\text{T}}\smla_l - \sum_{l \neq i} p_{il} \frac{1}{4} \nu \smlw_i^{\text{T}} \smla_l \smla_l \|_2
\leq O(s^2/(np)).
\end{aligned}
\end{equation}
\noindent where each terms is bounded as below
\begin{equation}
\begin{aligned}
\| \sum_{l \neq j \neq i} p_{ijl} \frac{1}{4} \nu \smlw_i^{\text{T}} \smla_l \smlw_j \smlw_j^{\text{T}} \smla_l \|_2 = \frac{1}{4} \nu  (s^3/p^3) \| \W_{\backslash i} \|_2 \leq O(s^3/(pn\sqrt{n})).
\end{aligned}
\end{equation}
\noindent where $z_j = \sum_{l \neq i,j} \smlw_i^{\text{T}} \smla_l\smlw_j^{\text{T}} \smla_l $, hence, $\| z \|_2 \leq O(\frac{p\sqrt{p}}{n})$.
\begin{equation}
\begin{aligned}
\| \sum_{j \neq i} p_{ij} \frac{1}{4} \nu \smlw_i^{\text{T}} \smla_j \smlw_j \smlw_j^{\text{T}}\smla_j \|_2 = \| \frac{1}{4} \nu (s^2/p^2) \| \W_{\backslash i} z \|_2 \leq O(s^2/(np)).
\end{aligned}
\end{equation}
\noindent where $z_j = \smlw_i^{\text{T}} \smla_j \smlw_j^{\text{T}}\smla_j$, hence, $\| z \|_2 \leq O(\frac{\sqrt{p}}{\sqrt{n}})$.
\begin{equation}
\begin{aligned}
\| \sum_{l \neq j \neq i} p_{ijl} \frac{1}{4} \nu \smlw_i^{\text{T}} \smla_l \smlw_j \smlw_j^{\text{T}} \smla_l \|_2 = \frac{1}{4} \nu  (s^3/p^3) \| \W_{\backslash i} \|_2 \leq O(s^3/(pn\sqrt{n})).
\end{aligned}
\end{equation}
\begin{equation}
\begin{aligned}
\| p_{il} \frac{1}{4} \nu \smlw_i^{\text{T}} \smla_l \smlw_i \smlw_i^{\text{T}}\smla_l \|_2  \leq O(s^2/(np^2)).
\end{aligned}
\end{equation}
\noindent Following a similar approach for $r_2$, we get
\begin{equation}
\begin{aligned}
\| E[r_2] \|_2 &= \| E[\sum_{l \in S^{\backslash i}} \frac{1}{4} \nu \smla_l \smlw_i^{\text{T}} (\W_S\W_S^{\text{T}} - \eye) \smla_l] \|_2 = \| E[ \sum_{l \in S^{\backslash i}} \frac{1}{4} \nu \smla_l \smlw_i^{\text{T}} (\W_S\W_S^{\text{T}})\smla_l - \sum_{l \in S^{\backslash i}} \frac{1}{4} \nu \smla_l \smlw_i^{\text{T}} \smla_l] \|_2\\
&= \| \sum_{l \neq j \neq i} p_{ijl} \frac{1}{4} \nu \smla_l \smlw_i^{\text{T}} \smlw_j \smlw_j^{\text{T}} \smla_l +  \sum_{j \neq i} p_{ij} \frac{1}{4} \nu \smla_j \smlw_i^{\text{T}} \smlw_j \smlw_j^{\text{T}}\smla_j + p_{il} \frac{1}{4} \nu \smla_l \smlw_i^{\text{T}} \smlw_i \smlw_i^{\text{T}}\smla_l - \sum_{l \neq i} p_{il} \frac{1}{4} \nu \smlw_i^{\text{T}} \smla_l \smla_l \|_2\\
&= \| \sum_{l \neq j \neq i} p_{ijl} \frac{1}{4} \nu \smla_l \smlw_i^{\text{T}} \smlw_j \smlw_j^{\text{T}} \smla_l +  \sum_{j \neq i} p_{ij} \frac{1}{4} \nu \smla_j \smlw_i^{\text{T}} \smlw_j \smlw_j^{\text{T}}\smla_j \|_2
\leq O(s^2/(np)).
\end{aligned}
\end{equation}
\noindent Next, we bound $\| \boldeps \|_2$. We know that Hessian of sigmoid is bounded (i.e., $\| \nabla^2\sigma(u_t) \|_2 \leq C \approx 0.1$). We denote row $t$ of the matrix $\A$ by $\tilde \smla_t$.
\begin{equation}
\begin{aligned}
\| \boldeps \|_2 &\leq \sum_{t=1}^n \| (\tilde \smla^{\text{T}}_{t,S} \x_S^*)^2 \nabla^2 \sigma(u_t) \|_2 \leq \sum_{t=1}^n \| \tilde \smla^{\text{T}}_{t,S} \x_S^* \|_2^2 \| \nabla^2 \sigma(u_t) \|_2
\leq \| \A_S \|_2^2 \|\x_S^* \|_2^2 \| \nabla^2 \sigma(u_t) \|_2 \leq O(C_x^2 s).
\end{aligned}
\end{equation}
\noindent Following a similar approach, we get
\begin{equation}
\begin{aligned}
\| \tilde \boldeps \|_2 &\leq \sum_{t=1}^n \| [4\W_S \W_S^{\text{T}} (\mu - \frac{1}{2})  - \W_S \bvec_S]_t^2 \nabla^2 \sigma(u_t) \|_2
\leq (\| 4\W_S \W_S^{\text{T}} (\mu - \frac{1}{2}) \|_2^2 + \| \W_S \bvec_S\|_2^2) \| \nabla^2 \sigma(u_t) \|_2\\
&\leq O(C \| \W_S \|_2^2 \| \A_S \x_S^* + \boldeps \|_2^2) \leq O(C_x^2 s)
\end{aligned}
\end{equation}
\noindent So,
\begin{equation}
\| \smlw_i^{\text{T}} \boldeps (\W_S\W_S^{\text{T}} - \eye) \A\x^* \|_2 \leq \| \smlw_i \|_2 \| \boldeps \|_2 (\|\W_S^{\text{T}} \|_2^2 +1) \| \A_S \|_2 \| \x_S^* \|_2 \leq O(C_x^3 s \sqrt{s}).
\end{equation}
\noindent Hence, 
\begin{equation}
\| E[e_1] \|_2 \leq O(C_x^3 s \sqrt{s}).
\end{equation}
\noindent Similarly, we have $\| E[e_2] \|_2 \leq O(C_x^3 s \sqrt{s})$.
\begin{equation}
\begin{aligned}
&\|  \left(( \A\x^* + \boldeps) \smlw_i^{\text{T}} + \smlw_i^{\text{T}} (\A\x^* + \boldeps) \eye \right) ((\W_S\W_S^{\text{T}} - \eye) \boldeps + \tilde \boldeps) \|_2\\
&\leq 2(\| \A_S \|_2 \| \x_S^* \|_2  + \| \boldeps \|_2) \| \smlw_i \|_2 \left( (\|\W_S^{\text{T}} \|_2^2 + 1) \| \boldeps \|_2 + \| \tilde \boldeps \|_2\right)\\
&= O((\| \x_S^* \|_2  + \| \boldeps \|_2) \| \boldeps \|_2) \leq O(\max(C_x^3 s \sqrt{s}, C_x^4 s^2)).
\end{aligned}
\end{equation}
\noindent Hence,
\begin{equation}
\| E[e_3] \|_2 \leq O(\max(C_x^3 s \sqrt{s}, C_x^4 s^2)).
\end{equation}
\noindent We have
\begin{equation}
\begin{aligned}
&\| (-b_i) ((\W_S\W_S^{\text{T}} - \eye) \boldeps + \tilde \boldeps) \|_2 \leq O(C_x^2 s)
\end{aligned}
\end{equation}
\noindent Hence,
\begin{equation}
\| E[e_4] \|_2 \leq O(\max(C_x^2 s).
\end{equation}
\noindent Using the above bounds, we have
\begin{equation}
\| \beta \|_2 \leq O(\max(C_x^3 s \sqrt{s}, C_x^4 s^2)).
\end{equation}
\noindent Using (A14), we get
\begin{equation}
\begin{aligned}
\| v \|_2 &= \| \frac{1}{4} p_i (\nu \tau_i ( \tau_i - 1) + b_i^2) \smlw_i + \zeta + \beta \|_2
\leq \frac{1}{4} p_i |\nu \tau_i (\tau_i - 1) + b_i^2| \|\smlw_i \|_2 + \| \zeta \|_2 + \| \beta \|_2\\
&\leq \frac{1}{4} p_i (2 \nu \tau_i (1 - \tau_i)) + \| \zeta \|_2 + \| \beta \|_2
\leq \frac{1}{4} p_i \nu \tau_i q \| \smlw_i - \smla_i \|_2 + \| \zeta \|_2 + \| \beta \|_2\\
&\leq \frac{1}{4} p_i \nu \tau_i q \| \smlw_i - \smla_i \|_2  + O(\max(C_x^3 s \sqrt{s}, C_x^4 s^2)).
\end{aligned}
\end{equation}
\end{proof}
\begin{lemma}\label{lemma:dir_relu}
Suppose the generative model satisfies $(A1) - (A14)$. Then
\begin{equation}
2 \langle g_{i}, \smlw_i - \smla_i \rangle \geq (\frac{1}{4} \nu \tau_i s/p) (1 -  2q^2) \| \smlw_i - \smla_i \|^2_2 + \frac{1}{(\frac{1}{4} \nu \tau_i s/p)} \| g_{i} \|_2^2 - O(C_x^6 p\max(s^2/ \tau_i, C_x^2 s^3 / \tau_i)).
\end{equation}
\end{lemma}
\begin{proof} From Lemma~\ref{lemma:v_relu}, we have
\begin{equation}
\| v \|_2 \leq \frac{1}{4} p_i \tau_i \nu q \| \smlw_i - \smla_i \|_2  + O(\max(C_x^3 s \sqrt{s}, C_x^4 s^2)).
\end{equation}
\noindent Hence,
\begin{equation}
\| v \|_2^2 \leq 2(\frac{1}{4} \tau_i \nu q s/p)^2 \| \smlw_i - \smla_i \|_2^2  + O(\max(C_x^6 s^3, C_x^8 s^4)).
\end{equation}
\noindent We have $g_i = \frac{1}{4} p_i \nu \tau_i (\smlw_i - \smla_i)  + v$. Taking the norm,
\begin{equation}
 \| g_i \|_2^2 = (\frac{1}{4} p_i \nu \tau_i)^2  \| \smlw_i - \smla_i \|_2^2  +  \| v \|_2^2 + 2 (\frac{1}{4} p_i \nu \tau_i) \langle v, \smlw_i - \smla_i \rangle.
\end{equation}
\begin{equation}
2 \langle v, \smlw_i - \smla_i \rangle =  - (\frac{1}{4} p_i \nu \tau_i)  \| \smlw_i - \smla_i \|_2^2 + \frac{1}{(\frac{1}{4} p_i \nu \tau_i)} \| g_i \|_2^2 - \frac{1}{(\frac{1}{4} p_i \nu \tau_i)} \| v \|_2^2.
\end{equation}
\begin{equation}
\begin{aligned}
&2 \langle g_{i}, \smlw_i - \smla_i \rangle = \frac{1}{4} p_i \nu \tau_i \| \smlw_i - \smla_i \|^2_2 + \frac{1}{(\frac{1}{4} p_i \nu \tau_i)} \| g_{i} \|_2^2 - \frac{1}{(\frac{1}{4} p_i \nu \tau_i)} \| v \|_2^2\\
&\geq (\frac{1}{4} \nu \tau_i s/p) (1 -  2q^2) \| \smlw_i - \smla_i \|^2_2 + \frac{1}{(\frac{1}{4} \nu \tau_i s/p)} \| g_{i} \|_2^2 - O(C_x^6 p\max(s^2/ \tau_i, C_x^2 s^3 / \tau_i)).
\end{aligned}
\end{equation}
\end{proof}
\noindent Intuitively, Lemma \ref{lemma:dir_relu} suggests that the gradient is approximately along the same direction as $\smlw_i - \smla_i$, so at every iteration of the gradient descent, $\smlw_i$ gets closer and closer to $\smla_i$. Given Lemma \ref{lemma:dir_relu}, rigorously, from the descent property of Theorem 6 in~\cite{Arora15}, we can see that given the learning rate $\kappa = \max_i(\frac{1}{\frac{1}{4} \nu \tau_i s/p})$, letting $\delta = \kappa (\frac{1}{4} \nu \tau_i s/p) (1 - 2q^2) \in (0,1)$, we have the descent property as follows
\begin{equation}
\| \smlw_i^{(l+1)} - \smla_i \|_2^2 \leq (1- \delta) \| \smlw_i^{(l)} - \smla_i \|_2^2 + \kappa \cdot O(C_x^6 p\max(s^2/ \tau_i, C_x^2 s^3 / \tau_i)).
\end{equation}
\begin{lemma}\label{lemma:L_relu}
Suppose $\| \smlw_i^{(l+1)} - \smla_i \|_2^2 \leq (1- \delta) \| \smlw_i^{(l)} - \smla_i \|_2^2 + \kappa \cdot O(C_x^6 p\max(s^2/ \tau_i, C_x^2 s^3 / \tau_i))$ where $\delta = \kappa (\frac{1}{4} \nu \tau_i s/p) (1 - 2q^2) \in (0,1)$ and $O(\frac{C_x^4 p^2 \max(s, C_x^4 s^2)}{\tau_i^2 (1 -  2q^2)}) < \| \smlw_i^{(0)} - \smla_i \|_2^2$. Then
\begin{equation}
\begin{aligned}
\| \smlw_i^{(L)} - \smla_i \|_2^2 &\leq (1- \delta/2)^L \| \smlw_i^{(0)} - \smla_i \|_2^2.
\end{aligned}
\end{equation}
\end{lemma}
\begin{proof}
\noindent Performing the gradient update $L$ times, 
\begin{equation}
\begin{aligned}
\| \smlw_i^{(L)} - \smla_i \|_2^2 &\leq (1- \delta)^L \| \smlw_i^{(0)} - \smla_i \|_2^2 + \frac{1}{(\frac{1}{4} \nu \tau_i s/p) (1 -  2q^2)} O(C_x^6 p\max(s^2/ \tau_i, C_x^2 s^3 / \tau_i))\\
&\leq (1- \delta)^L \| \smlw_i^{(0)} - \smla_i \|_2^2 + O(\frac{C_x^6 p^2 \max(s, C_x^2 s^2)}{\tau_i^2 (1 -  2q^2)}).
\end{aligned}
\end{equation}
\noindent From Theorem 6 in~\cite{Arora15}, if $O(\frac{C_x^6 p^2 \max(s, C_x^2 s^2)}{\tau_i^2 (1 -  2q^2)}) < \| \smlw_i^{(0)} - \smla_i \|_2^2$, then we have
\begin{equation}
\begin{aligned}
\| \smlw_i^{(L)} - \smla_i \|_2^2 &\leq (1- \delta/2)^L \| \smlw_i^{(0)} - \smla_i \|_2^2.
\end{aligned}
\end{equation}
\end{proof}
\begin{corollary}\label{cor:1_relu}
Given (A2), the condition of Lemma~\ref{lemma:L_relu} is simplified to $O(\frac{\max(s,  s^2/ p^{\frac{2}{3} + 2\xi})}{p^{6\xi}\tau_i^2 (1 -  2q^2)}) < \| \smlw_i^{(0)} - \smla_i \|_2^2$.
\end{corollary}
\noindent The intuition behind the bound on the amplitude of $\x^*$ in (A2) is that as $C_x$ gets smaller, the range of $\sigma(\A\x^*)$ is concentrated around the linear region of the sigmoid function (i.e., around $\sigma(\mathbf{0})$); thus $\boldeps$, which is the difference between $\sigma(\A\x^*)$ and the linear region of sigmoid $\frac{1}{2} + \frac{1}{4} \A\x^*$, is smaller. Hence, the upper bound on $\| v \|_2$ would be smaller and $O(\frac{\max(s,  s^2/ p^{\frac{2}{3} + 2\xi})}{p^{6\xi}\tau_i^2 (1 -  2q^2)})$ would get smaller.
\subsection{Sparse coding with $\prox_{\bvec} = \text{HT}_{\bvec}$}
\subsubsection{Gradient derivation}
This is section, we derive $g_i$ for the case when $\prox_{\bvec} = \text{HT}_{\bvec}$ following a similar approach to the previous section.
\begin{equation}
\begin{aligned}
&\lim_{M \rightarrow \infty} \frac{\partial \mathcal{L}_{\W}}{\partial \smlw_i} = \lim_{M \rightarrow \infty} \frac{\partial \cvec_1}{\partial \smlw_i}  \frac{\partial \mathcal{L}_{\W}}{\partial \cvec_1} + \frac{\partial \cvec_{2}}{\partial \smlw_i}  \frac{\partial \mathcal{L}_{\W}}{\partial \cvec_2} = \frac{\partial \cvec_1}{\partial \smlw_i} \frac{\partial \x}{\partial \smlc_1} \frac{\partial \cvec_2}{\partial \x}\ \frac{\mathcal{L}_{\W}}{\partial \cvec_2}+ \frac{\partial \cvec_{2}}{\partial \smlw_i} \frac{\partial \mathcal{L}_{\W}}{\partial \cvec_2}\\
&= \left(\underbrace{[0, 0, \ldots, 4(\boldmu - \frac{1}{2}), \ldots, 0]}_{n \times p} \underbrace{\text{diag}(\prox_{\bvec}^{\prime}(\cvec_{1}))}_{p \times p} \W^{\text{T}} + \one_{\x_i \neq 0} \smlw_i^{\text{T}} 4(\boldmu - \frac{1}{2}) \eye \right) \left(- \sigma(\A\x^*) +  \sigma(4\W_S\W_S^{\text{T}} (\boldmu - \frac{1}{2}))\right)\\
&=  \left( \prox_{\bvec}^{\prime}(\cvec_{1,i}) 4(\boldmu - \frac{1}{2}) \smlw_i^{\text{T}} + \one_{\x_i \neq 0} \smlw_i^{\text{T}} 4(\boldmu - \frac{1}{2}) \eye \right) \left(\sigma(4\W_S\W_S^{\text{T}} (\boldmu -\frac{1}{2})) - \sigma(\A\x^*)\right).
\end{aligned}
\end{equation}
\noindent We further expand the gradient, by replacing $\sigma(.)$ with its Taylor expansion. We have
\begin{equation}
\sigma(\A\x^*) = \frac{1}{2} + \frac{1}{4} \A\x^* + \boldeps,
\end{equation}
\noindent where $\boldeps = [\epsilon_1, \ldots, \epsilon_n]^{\text{T}}$, $\epsilon_d = \nabla^2\sigma(u_d)([\A\x^*]_d)^2$, and  $0 \leq u_d \leq [\A\x^*]_d$. Similarly,
\begin{equation}
\sigma(4\W_S\W_S^{\text{T}} (\boldmu - \frac{1}{2})) = \frac{1}{2} + \W_S\W_S^{\text{T}} (\boldmu - \frac{1}{2}) + \tilde \boldeps,
\end{equation}
\noindent where $\tilde \boldeps = [\tilde \epsilon_1, \ldots, \tilde \epsilon_n]^{\text{T}}$, $\tilde \epsilon_d = \nabla^2\sigma(\tilde u_d) ([4\W_S\W_S^{\text{T}} (\boldmu - \frac{1}{2})]_d)^2$ , and $0 \leq \tilde u_d \leq [4\W_S\W_S^{\text{T}} (\boldmu - \frac{1}{2})]_d$. Again, replacing $\boldmu$ with Taylor expansion of $\sigma(\A\x^*)$, we get
\begin{equation}
\sigma(4\W_S\W_S^{\text{T}} (\boldmu - \frac{1}{2})) = \frac{1}{2} + \W_S\W_S^{\text{T}} (\frac{1}{4} \A\x^* + \boldeps) + \tilde \boldeps.
\end{equation}
\noindent By symmetry, $E[\boldeps \mid S] = E[\tilde \boldeps \mid S] = 0$. The expectation of gradient $g_i$ would be
\begin{equation}
g_i = E[\left( \one_{\x_i \neq 0} ( \A\x^* + 4 \boldeps) \smlw_i^{\text{T}} + \one_{\x_i \neq 0} \smlw_i^{\text{T}} (\A\x^* + 4 \boldeps) \eye \right) (\frac{1}{4} ( \W_S\W_S^{\text{T}} - \eye) (\A\x^*) + (\W_S\W_S^{\text{T}} - \eye) \boldeps + \tilde \boldeps)].
\end{equation}
\subsubsection{Gradient dynamics}
\noindent Given the code consistency from the forward pass of the encoder, we replace $\one_{\x_i \neq 0}$ with $\one_{\x_i^* \neq 0}$ and denote the error by $\gamma$ as below which is small for large $p$~\cite{Nguyen2019}.
\begin{equation}
\gamma = E[(\one_{\x_i^* \neq 0} - \one_{\x_i \neq 0}) \left( ( \A\x^* + 4 \boldeps) \smlw_i^{\text{T}} +\smlw_i^{\text{T}} (\A\x^* + 4 \boldeps) \eye \right) (\frac{1}{4} ( \W_S\W_S^{\text{T}} - \eye) (\A\x^*) + (\W_S\W_S^{\text{T}} - \eye) \boldeps + \tilde \boldeps)].
\end{equation}
\noindent Now, we write $g_i$ as
\begin{equation}
g_i = E[\one_{\x_i^* \neq 0} \left( ( \A\x^* + 4 \boldeps) \smlw_i^{\text{T}} + \smlw_i^{\text{T}} (\A\x^* + 4 \boldeps) \eye \right) (\frac{1}{4} ( \W_S\W_S^{\text{T}} - \eye) (\A\x^*) + (\W_S\W_S^{\text{T}} - \eye) \boldeps + \tilde \boldeps)] + \gamma.
\end{equation}
\noindent We can see that if $i \notin S$ then $\one_{\x_i^*} = 0$ hence, $g_i = 0$. Thus, in our analysis, we only consider the case $i \in S$. We decompose $g_i$ as below.
\begin{equation}
g_i = g_i^{(1)} + g_i^{(2)} + g_i^{(3)} + \gamma,
\end{equation}
\noindent where
\begin{equation}
\begin{aligned}
g_i^{(1)} &= E[\frac{1}{4} \smlw_i^{\text{T}} (\A\x^* + 4 \boldeps) (\W_S\W_S^{\text{T}} - \eye) \A\x^*].
\end{aligned}
\end{equation}
\begin{equation}
\begin{aligned}
g_i^{(2)} &= E[\frac{1}{4} (\A\x^* + 4 \boldeps) \smlw_i^{\text{T}} (\W_S\W_S^{\text{T}} - \eye) \A\x^*] .
\end{aligned}
\end{equation}
\begin{equation}
\begin{aligned}
g_i^{(3)} &= E[\left( ( \A\x^* + 4 \boldeps) \smlw_i^{\text{T}} + \smlw_i^{\text{T}} (\A\x^* + 4 \boldeps) \eye \right) ((\W_S\W_S^{\text{T}} - \eye) \boldeps + \tilde \boldeps)].
\end{aligned}
\end{equation}
\noindent We define
\begin{equation}
\begin{aligned}
g_{i,S}^{(1)} &= E[\frac{1}{4} \smlw_i^{\text{T}} (\A\x^* + 4 \boldeps) (\W_S\W_S^{\text{T}} - \eye) \A\x^* \mid S].
\end{aligned}
\end{equation}
\begin{equation}
\begin{aligned}
g_{i,S}^{(2)} &= E[\frac{1}{4} (\A\x^* + 4 \boldeps) \smlw_i^{\text{T}} (\W_S\W_S^{\text{T}} - \eye) \A\x^* \mid S].
\end{aligned}
\end{equation}
\begin{equation}
\begin{aligned}
g_{i,S}^{(3)} &= E[\left( ( \A\x^* + 4 \boldeps) \smlw_i^{\text{T}} + \smlw_i^{\text{T}} (\A\x^* + 4 \boldeps) \eye \right) ((\W_S\W_S^{\text{T}} - \eye) \boldeps + \tilde \boldeps) \mid S].
\end{aligned}
\end{equation}
\noindent Hence, $g_{i}^{(k)} = E[g_{i,S}^{(k)}]$ for $k=1, \ldots , 3$ where the expectations are with respect to the support S.
\begin{equation}
\begin{aligned}
g_{i,S}^{(1)} &= E[\frac{1}{4} \smlw_i^{\text{T}} (\A\x^* + 4 \boldeps) (\W_S\W_S^{\text{T}} - \eye) \A\x^* \mid S]\\
&= \sum_{j,l \in S}  E[\frac{1}{4} \smlw_i^{\text{T}} \smla_j \x_j^* (\W_S\W_S^{\text{T}} - \eye) \smla_l\x_l^* \mid S] + E[\smlw_i^{\text{T}} \boldeps (\W_S\W_S^{\text{T}} - \eye) \A\x^* \mid S]\\
&= \frac{1}{4} \nu \smlw_i^{\text{T}} \smla_i (\W_S\W_S^{\text{T}} - \eye) \smla_i + \sum_{l \in S^{\backslash i}} \frac{1}{4} \nu \smlw_i^{\text{T}} \smla_l (\W_S\W_S^{\text{T}} - \eye) \smla_l + e_1
= \frac{1}{4} \nu \smlw_i^{\text{T}} \smla_i (\W_S\W_S^{\text{T}} - \eye) \smla_i + r_1 + e_1.
\end{aligned}
\end{equation}
\noindent We denote $r_1 = \sum_{l \in S^{\backslash i}} \frac{1}{4} \nu \smlw_i^{\text{T}} \smla_l (\W_S\W_S^{\text{T}} - \eye) \smla_l$ and $e_1 = E[\smlw_i^{\text{T}} \boldeps (\W_S\W_S^{\text{T}} - \eye) \A\x^* \mid S]$. Similarly for $g_{i,S}^{(2)}$, we have
\begin{equation}
g_{i,S}^{(2)} = \frac{1}{4} \nu \smla_i \smlw_i^{\text{T}} (\W_S\W_S^{\text{T}} - \eye) \smla_i  + r_2 + e_2.
\end{equation}
\noindent We denote $r_2 = \sum_{l \in S^{\backslash i}} \frac{1}{4} \nu \smla_l \smlw_i^{\text{T}} (\W_S\W_S^{\text{T}} - \eye) \smla_l$, $e_2 = E[\boldeps \smlw_i^{\text{T}} (\W_S\W_S^{\text{T}} - \eye) \A\x^* \mid S]$, and $e_3 = g_{i,S}^{(3)}$. We also denote $\beta = E[r_1 + r_2 + e_1 + e_2 + e_3] + \gamma$. Combining the terms,
\begin{equation}
\begin{aligned}
g_i &= E[\frac{1}{4} \nu \smlw_i^{\text{T}} \smla_i (\W_S\W_S^{\text{T}} - \eye) \smla_i + \frac{1}{4} \nu \smla_i \smlw_i^{\text{T}}(\W_S\W_S^{\text{T}} - \eye) \smla_i] + \beta \\
&= E[- \frac{1}{2} \nu \tau_i \smla_i + \frac{1}{4} \nu \tau_i \sum_{j \in S} \smlw_j \smlw_j^{\text{T}} \smla_i + \frac{1}{4} \nu \smla_i \smlw_i^{\text{T}} \sum_{j \in S} \smlw_j \smlw_j^{\text{T}} \smla_i] + \beta \\
&= E[- \frac{1}{2} \nu \tau_i \smla_i + \frac{1}{4} \nu \tau_i^2 \smlw_i + \frac{1}{4} \nu \tau_i \sum_{j \in S^{\backslash i}} \smlw_j \smlw_j^{\text{T}} \smla_i +  \frac{1}{4} \nu \tau_i  \| \smlw_i\|_2^2 \smla_i + \frac{1}{4} \nu \left(\smla_i \smlw_i^{\text{T}} \right) \sum_{j \in S^{\backslash i}} \smlw_j \smlw_j^{\text{T}} \smla_i] + \beta \\
&= - \frac{1}{4} p_i \nu \tau_i \smla_i + p_i \frac{1}{4} \nu \tau_i^2 \smlw_i + \zeta + \beta,
\end{aligned}
\end{equation}
\noindent where $\zeta = \sum_{j \in [p]^{\backslash i}} \frac{1}{4} p_{ij} \nu \tau_i \smlw_j \smlw_j^{\text{T}} \smla_i + \frac{1}{4} p_{ij} \nu \smla_i \smlw_i^{\text{T}} \smlw_j \smlw_j^{\text{T}} \smla_i$. We continue
\begin{equation}
g_i = \frac{1}{4} p_i \nu \tau_i (\smlw_i - \smla_i)  + v,
\end{equation}
\noindent where we denote $v =\frac{1}{4} p_i \nu \tau_i (\tau_i - 1) \smlw_i + \zeta + \beta$.
\begin{lemma}\label{lemma:v}
Suppose the generative model satisfies $(A1) - (A13)$. Then
\begin{equation}
\| v \|_2 \leq \frac{1}{8} p_i \nu \tau_i q \| \smlw_i - \smla_i \|_2  + O(\max(C_x^3 s \sqrt{s}, C_x^4 s^2))
\end{equation}
\end{lemma}
\begin{proof}
\begin{equation}
\begin{aligned}
\| \zeta \|_2  &=  \| \sum_{j \in [p]^{\backslash i}} \frac{1}{4} p_{ij} \nu \tau_i \smlw_j \smlw_j^{\text{T}} \smla_i + \frac{1}{4} p_{ij} \nu \smla_i \smlw_i^{\text{T}} \smlw_j \smlw_j^{\text{T}} \smla_i \|_2
= \| \frac{1}{4} p_{ij} \nu \tau_i  \W_{\backslash i} \W_{\backslash i}^{\text{T}} \smla_i + \frac{1}{4} p_{ij} \nu \smla_i \smlw_i^{\text{T}} \W_{\backslash i} \W_{\backslash i}^{\text{T}} \smla_i \|_2 \\
&\leq \frac{1}{4} p_{ij} \nu \tau_i \| \W_{\backslash i} \|_2^2 \| \smla_i\|_2 + \frac{1}{4} p_{ij} \nu \| \smlw_i \|_2 \| \W_{\backslash i}\|_2^2 \| \smla_i\|_2^2
= \frac{1}{4} O(\nu \tau_i s^2 / (np)) + \frac{1}{4} O(\nu s^2 / (np)) = O(s^2/(np)).
\end{aligned}
\end{equation}
\begin{equation}
\begin{aligned}
&\| E[ r_1] \|_2 = \| E[\sum_{l \in S^{\backslash i}} \frac{1}{4} \nu \smlw_i^{\text{T}} \smla_l ( \W_S\W_S^{\text{T}} - \eye) \smla_l] \|_2
= \| E[\sum_{l \in S^{\backslash i}} \frac{1}{4} \nu \smlw_i^{\text{T}} \smla_l  \W_S\W_S^{\text{T}} \smla_l - \sum_{l \in S^{\backslash i}} \frac{1}{4} \nu \smlw_i^{\text{T}} \smla_l \smla_l] \|_2 \\
=& \| \sum_{l \neq j \neq i} p_{ijl} \frac{1}{4} \nu \smlw_i^{\text{T}} \smla_l \smlw_j \smlw_j^{\text{T}} \smla_l +  \sum_{j \neq i} p_{ij} \frac{1}{4} \nu \smlw_i^{\text{T}} \smla_j \smlw_j \smlw_j^{\text{T}}\smla_j + p_{il} \frac{1}{4} \nu \smlw_i^{\text{T}} \smla_l \smlw_i \smlw_i^{\text{T}}\smla_l - \sum_{l \neq i} p_{il} \frac{1}{4} \nu \smlw_i^{\text{T}} \smla_l \smla_l \|_2
\leq O(s^2/(np)).
\end{aligned}
\end{equation}
\noindent where each terms is bounded as below
\begin{equation}
\begin{aligned}
\| \sum_{l \neq j \neq i} p_{ijl} \frac{1}{4} \nu \smlw_i^{\text{T}} \smla_l \smlw_j \smlw_j^{\text{T}} \smla_l \|_2 = \frac{1}{4} \nu  (s^3/p^3) \| \W_{\backslash i} \|_2 \leq O(s^3/(pn\sqrt{n})).
\end{aligned}
\end{equation}
\noindent where $z_j = \sum_{l \neq i,j} \smlw_i^{\text{T}} \smla_l\smlw_j^{\text{T}} \smla_l $, hence, $\| z \|_2 \leq O(\frac{p\sqrt{p}}{n})$.
\begin{equation}
\begin{aligned}
\| \sum_{j \neq i} p_{ij} \frac{1}{4} \nu \smlw_i^{\text{T}} \smla_j \smlw_j \smlw_j^{\text{T}}\smla_j \|_2 = \| \frac{1}{4} \nu (s^2/p^2) \| \W_{\backslash i} z \|_2 \leq O(s^2/(np)).
\end{aligned}
\end{equation}
\noindent where $z_j = \smlw_i^{\text{T}} \smla_j \smlw_j^{\text{T}}\smla_j$, hence, $\| z \|_2 \leq O(\frac{\sqrt{p}}{\sqrt{n}})$.
\begin{equation}
\begin{aligned}
\| \sum_{l \neq j \neq i} p_{ijl} \frac{1}{4} \nu \smlw_i^{\text{T}} \smla_l \smlw_j \smlw_j^{\text{T}} \smla_l \|_2 = \frac{1}{4} \nu  (s^3/p^3) \| \W_{\backslash i} \|_2 \leq O(s^3/(pn\sqrt{n})).
\end{aligned}
\end{equation}
\begin{equation}
\begin{aligned}
\| p_{il} \frac{1}{4} \nu \smlw_i^{\text{T}} \smla_l \smlw_i \smlw_i^{\text{T}}\smla_l \|_2  \leq O(s^2/(np^2)).
\end{aligned}
\end{equation}
\noindent Following a similar approach for $r_2$, we get
\begin{equation}
\begin{aligned}
\| E[r_2] \|_2 &= \| E[\sum_{l \in S^{\backslash i}} \frac{1}{4} \nu \smla_l \smlw_i^{\text{T}} (\W_S\W_S^{\text{T}} - \eye) \smla_l] \|_2 = \| E[ \sum_{l \in S^{\backslash i}} \frac{1}{4} \nu \smla_l \smlw_i^{\text{T}} (\W_S\W_S^{\text{T}})\smla_l - \sum_{l \in S^{\backslash i}} \frac{1}{4} \nu \smla_l \smlw_i^{\text{T}} \smla_l] \|_2\\
&= \| \sum_{l \neq j \neq i} p_{ijl} \frac{1}{4} \nu \smla_l \smlw_i^{\text{T}} \smlw_j \smlw_j^{\text{T}} \smla_l +  \sum_{j \neq i} p_{ij} \frac{1}{4} \nu \smla_j \smlw_i^{\text{T}} \smlw_j \smlw_j^{\text{T}}\smla_j + p_{il} \frac{1}{4} \nu \smla_l \smlw_i^{\text{T}} \smlw_i \smlw_i^{\text{T}}\smla_l - \sum_{l \neq i} p_{il} \frac{1}{4} \nu \smlw_i^{\text{T}} \smla_l \smla_l \|_2\\
&= \| \sum_{l \neq j \neq i} p_{ijl} \frac{1}{4} \nu \smla_l \smlw_i^{\text{T}} \smlw_j \smlw_j^{\text{T}} \smla_l +  \sum_{j \neq i} p_{ij} \frac{1}{4} \nu \smla_j \smlw_i^{\text{T}} \smlw_j \smlw_j^{\text{T}}\smla_j \|_2
\leq O(s^2/(np)).
\end{aligned}
\end{equation}
\noindent Next, we bound $\| \boldeps \|_2$. We know that Hessian of sigmoid is bounded (i.e., $\| \nabla^2\sigma(u_t) \|_2 \leq C \approx 0.1$). We denote row $t$ of the matrix $\A$ by $\tilde \smla_t$.
\begin{equation}
\begin{aligned}
\| \boldeps \|_2 &\leq \sum_{t=1}^n \| (\tilde \smla^{\text{T}}_{t,S} \x_S^*)^2 \nabla^2 \sigma(u_t) \|_2 \leq \sum_{t=1}^n \| \tilde \smla^{\text{T}}_{t,S} \x_S^* \|_2^2 \| \nabla^2 \sigma(u_t) \|_2
\leq \| \A_S \|_2^2 \|\x_S^* \|_2^2 \| \nabla^2 \sigma(u_t) \|_2 \leq O(C_x^2 s).
\end{aligned}
\end{equation}
\noindent Following a similar approach, we get
\begin{equation}
\begin{aligned}
\| \tilde \boldeps \|_2 &\leq \sum_{t=1}^n \| [4\W_S \W_S^{\text{T}} (\mu - \frac{1}{2})]_t^2 \nabla^2 \sigma(u_t) \|_2 \leq \| 4\W_S \W_S^{\text{T}} (\mu - \frac{1}{2}) \|_2^2 \| \nabla^2 \sigma(u_t) \|_2\\
&\leq O(C \| \W_S \|_2^2 \| \A_S \x_S^* + \boldeps \|_2^2) \leq O(C_x^2 s)
\end{aligned}
\end{equation}
\noindent So,
\begin{equation}
\| \smlw_i^{\text{T}} \boldeps (\W_S\W_S^{\text{T}} - \eye) \A\x^* \|_2 \leq \| \smlw_i \|_2 \| \boldeps \|_2 (\|\W_S^{\text{T}} \|_2^2 +1) \| \A_S \|_2 \| \x_S^* \|_2 \leq O(C_x^3 s \sqrt{s}).
\end{equation}
\noindent Hence, 
\begin{equation}
\| E[e_1] \|_2 \leq O(C_x^3 s \sqrt{s}).
\end{equation}
\noindent Similarly, we have $\| E[e_2] \|_2 \leq O(C_x^3 s \sqrt{s})$.
\begin{equation}
\begin{aligned}
&\|  \left(( \A\x^* + \boldeps) \smlw_i^{\text{T}} + \smlw_i^{\text{T}} (\A\x^* + \boldeps) \eye \right) ((\W_S\W_S^{\text{T}} - \eye) \boldeps + \tilde \boldeps) \|_2\\
&\leq 2(\| \A_S \|_2 \| \x_S^* \|_2  + \| \boldeps \|_2) \| \smlw_i \|_2 \left( (\|\W_S^{\text{T}} \|_2^2 + 1) \| \boldeps \|_2 + \| \tilde \boldeps \|_2\right)\\
&= O((\| \x_S^* \|_2  + \| \boldeps \|_2) \| \boldeps \|_2) \leq O(\max(C_x^3 s \sqrt{s}, C_x^4 s^2)).
\end{aligned}
\end{equation}
\noindent Hence,
\begin{equation}
\| E[e_3] \|_2 \leq O(\max(C_x^3 s \sqrt{s}, C_x^4 s^2)).
\end{equation}
\noindent Using the above bounds, we have
\begin{equation}
\| \beta \|_2 \leq O(\max(C_x^3 s \sqrt{s}, C_x^4 s^2)).
\end{equation}
\noindent Hence,
\begin{equation}
\begin{aligned}
\| v \|_2 &= \| \frac{1}{4} p_i \nu \tau_i ( \tau_i - 1) \smlw_i + \zeta + \beta \|_2
\leq \frac{1}{4} p_i \nu \tau_i | (\tau_i - 1) | \|\smlw_i \|_2 + \| \zeta \|_2 + \| \beta \|_2\\
&\leq \frac{1}{4} p_i \nu \tau_i (\frac{1}{2} q \| \smlw_i - \smla_i \|_2) + \| \zeta \|_2 + \| \beta \|_2
\leq \frac{1}{8} p_i \nu \tau_i q \| \smlw_i - \smla_i \|_2  + O(\max(C_x^3 s \sqrt{s}, C_x^4 s^2)).
\end{aligned}
\end{equation}
\end{proof}
\begin{lemma}\label{lemma:dir}
Suppose the generative model satisfies $(A1) - (A13)$. Then
\begin{equation}
2 \langle g_{i}, \smlw_i - \smla_i \rangle \geq (\frac{1}{4} \nu \tau_i s/p) (1 -  \frac{q^2}{2}) \| \smlw_i - \smla_i \|^2_2 + \frac{1}{(\frac{1}{4} \nu \tau_i s/p)} \| g_{i} \|_2^2 - O(C_x^6 p\max(s^2/ \tau_i, C_x^2 s^3 / \tau_i)).
\end{equation}
\end{lemma}
\begin{proof} From Lemma~\ref{lemma:v}, we have
\begin{equation}
\| v \|_2 \leq \frac{1}{8} p_i \nu \tau_i q \| \smlw_i - \smla_i \|_2  + O(\max(C_x^3 s \sqrt{s}, C_x^4 s^2)).
\end{equation}
\noindent Hence,
\begin{equation}
\| v \|_2^2 \leq 2(\frac{1}{8} \tau_i \nu q s/p)^2 \| \smlw_i - \smla_i \|_2^2  + O(\max(C_x^6 s^3, C_x^8 s^4)).
\end{equation}
\noindent We have $g_i = \frac{1}{4} p_i \nu \tau_i (\smlw_i - \smla_i)  + v$. Taking the norm,
\begin{equation}
 \| g_i \|_2^2 = (\frac{1}{4} p_i \nu \tau_i)^2  \| \smlw_i - \smla_i \|_2^2  +  \| v \|_2^2 + 2 (\frac{1}{4} p_i \nu \tau_i) \langle v, \smlw_i - \smla_i \rangle.
\end{equation}
\begin{equation}
2 \langle v, \smlw_i - \smla_i \rangle =  - (\frac{1}{4} p_i \nu \tau_i)  \| \smlw_i - \smla_i \|_2^2 + \frac{1}{(\frac{1}{4} p_i \nu \tau_i)} \| g_i \|_2^2 - \frac{1}{(\frac{1}{4} p_i \nu \tau_i)} \| v \|_2^2.
\end{equation}
\begin{equation}
\begin{aligned}
&2 \langle g_{i}, \smlw_i - \smla_i \rangle = \frac{1}{4} p_i \nu \tau_i \| \smlw_i - \smla_i \|^2_2 + \frac{1}{(\frac{1}{4} p_i \nu \tau_i)} \| g_{i} \|_2^2 - \frac{1}{(\frac{1}{4} p_i \nu \tau_i)} \| v \|_2^2\\
&\geq (\frac{1}{4} \nu \tau_i s/p) (1 -  \frac{q^2}{2}) \| \smlw_i - \smla_i \|^2_2 + \frac{1}{(\frac{1}{4} \nu \tau_i s/p)} \| g_{i} \|_2^2 - O(C_x^6 p\max(s^2/ \tau_i, C_x^2 s^3 / \tau_i)).
\end{aligned}
\end{equation}
\end{proof}
\noindent Lemma \ref{lemma:dir} suggests that the gradient is approximately along the same direction as $\smlw_i - \smla_i$, so at every iteration of the gradient descent, $\smlw_i$ gets closer and closer to $\smla_i$. Given Lemma \ref{lemma:dir} from the descent property of Theorem 6 in~\cite{Arora15}, we can see that given the learning rate $\kappa = \max_i(\frac{1}{\frac{1}{4} \nu \tau_i s/p})$, letting $\delta = \kappa (\frac{1}{4} \nu \tau_i s/p) (1 - \frac{q^2}{2}) \in (0,1)$, we have the descent property as follows
\begin{equation}
\| \smlw_i^{(l+1)} - \smla_i \|_2^2 \leq (1- \delta) \| \smlw_i^{(l)} - \smla_i \|_2^2 + \kappa \cdot O(C_x^6 p\max(s^2/ \tau_i, C_x^2 s^3 / \tau_i)).
\end{equation}
\begin{lemma}\label{lemma:L}
Suppose $\| \smlw_i^{(l+1)} - \smla_i \|_2^2 \leq (1- \delta) \| \smlw_i^{(l)} - \smla_i \|_2^2 + \kappa \cdot O(C_x^6 p\max(s^2/ \tau_i, C_x^2 s^3 / \tau_i))$ where $\delta = \kappa (\frac{1}{4} \nu \tau_i s/p) (1 - \frac{q^2}{2}) \in (0,1)$ and $O(\frac{C_x^4 p^2 \max(s, C_x^4 s^2)}{\tau_i^2 (1 -  \frac{q^2}{2})}) < \| \smlw_i^{(0)} - \smla_i \|_2^2$. Then
\begin{equation}
\begin{aligned}
\| \smlw_i^{(L)} - \smla_i \|_2^2 &\leq (1- \delta/2)^L \| \smlw_i^{(0)} - \smla_i \|_2^2.
\end{aligned}
\end{equation}
\end{lemma}
\begin{proof}
\noindent Performing the gradient update $L$ times, 
\begin{equation}
\begin{aligned}
\| \smlw_i^{(L)} - \smla_i \|_2^2 &\leq (1- \delta)^L \| \smlw_i^{(0)} - \smla_i \|_2^2 + \frac{1}{(\frac{1}{4} \nu \tau_i s/p) (1 -  \frac{q^2}{2})} O(C_x^6 p\max(s^2/ \tau_i, C_x^2 s^3 / \tau_i))\\
&\leq (1- \delta)^L \| \smlw_i^{(0)} - \smla_i \|_2^2 + O(\frac{C_x^6 p^2 \max(s, C_x^2 s^2)}{\tau_i^2 (1 -  \frac{q^2}{2})}).
\end{aligned}
\end{equation}
\noindent From Theorem 6 in~\cite{Arora15}, if $O(\frac{C_x^6 p^2 \max(s, C_x^2 s^2)}{\tau_i^2 (1 -  \frac{q^2}{2})}) < \| \smlw_i^{(0)} - \smla_i \|_2^2$, then we have
\begin{equation}
\begin{aligned}
\| \smlw_i^{(L)} - \smla_i \|_2^2 &\leq (1- \delta/2)^L \| \smlw_i^{(0)} - \smla_i \|_2^2.
\end{aligned}
\end{equation}
\end{proof}
\begin{corollary}\label{cor:1}
Given (A2), the condition of Lemma~\ref{lemma:L} is simplified to $O(\frac{\max(s,  s^2/ p^{\frac{2}{3} + 2\xi})}{p^{6\xi}\tau_i^2 (1 -  \frac{q^2}{2})}) < \| \smlw_i^{(0)} - \smla_i \|_2^2$.
\end{corollary}
\section{BCOMP algorithm}
We implement binomial convolutional orthogonal matching pursuit (BCOMP) as a baseline for ECDL task, as mentioned in the Experiments section. BCOMP solves Eq.~(2) with $\ell_0$ psuedo-norm $\lVert \x^j \rVert_0$, instead of $\lVert \x^j \rVert_1$, and combines the idea of convolutional greedy pursuit~\cite{Mailhe2011} and binomial greedy pursuit~\cite{Vincent2002, Lozano2011}. BCOMP is a computationally efficient algorithm for ECDL, as 1) the greedy algorithms are generally considered faster than algorithms for $\ell_1$-regularized problems~\cite{TROP2007} and 2) it exploits the localized nature of $\smlh_c$ to speed up the computation of both CSC and CDU steps.\\

\noindent The superscript $g$ refers to one iteration of the the alternating-minimization procedure, for $g=1,\cdots,G$. We assume sparsity level of $T$ for BCOMP, which means that there are at most $T$ non-zeros values for $\x^j$, set differently according to the application. The subscript $t$ refers to a single iteration of the CSC step, where additional support for $\x^j$ is identified. The set $\mathcal{R}_t$ contains indices of the columns from $\Bigh$ that were chosen up to iteration $t$. The notation $\Bigh_i$ refers to the $i^{\text{th}}$ column of $\Bigh$. The index $n^j_{c,i}$ denotes the occurrence of the $i^{\text{th}}$ event from filter $c$ (the nonzero entries of $\x^j$ corresponding to filter $c$) in the $j^{\text{th}}$ observation. The optimization problems in line 10 and 17 are both constrained convex optimization problems that can be solved using standard convex programming packages.\\

\begin{algorithm}
\caption{ECDL by BCOMP}\label{alg:ECOMP}
\begin{algorithmic}
	\INPUT{$\{\y^{j}\}_{j=1}^{J}\in\mathbb{R}^N$, $\{\smlh_c^{(0)}\}_{c=1}^C\in\mathbb{R}^K$}
	\OUTPUT{ $\{\x^{j,(G)}\}_{j=1}^J\in \mathbb{R}^{C(N-K+1)}$, $\{\smlh_c^{(G)}\}_{c=1}^C\in\mathbb{R}^K$}
	\FOR{$g=1$ to$G$}
		\STATE (\textit{CSC} step)\;
		\FOR{$j=1$ to $J$}
			\STATE $\mathcal{R}_0=\emptyset$, $\x_1^{j,(g-1)}=\mathbf{0}$\;
			\FOR{$t=1$ to $T$}
				\STATE $\widetilde{\mathbf{y}}_t^{j}=\mathbf{y}_t^{j}-f^{-1}\big(\Bigh^{(g-1)}\x_{t-1}^{j,(g-1)}\big)$\;
				\STATE $c^{\ast},n^{\ast}=\argmax_{c,n} \{(\smlh_c^{(g-1)}\star \widetilde{\mathbf{y}}_t^{j})[n]\}_{c,n=1}^{C,N-K+1}$\;
				\STATE $i = c^{\ast}(N-K+1)+n^{\ast}$\;
				\STATE $\mathcal{R}_t = \mathcal{R}_{t-1}\cup \Bigh_i^{(g-1)}$\;
				\STATE $\x^{j,(g)}_t=\argmin_{\x^j} -\log p(\y^j\vert \{\smlh_c^{(g-1)}\}_{c=1}^C,\x^j)$, s.t. $\begin{cases}
				\x^j[n]\geq 0 \text{ for } n\in \mathcal{R}_t\\
			\x^j[n]=0 \text{ for } n\notin \mathcal{R}_t\\
				\end{cases}$  
			\ENDFOR
	\ENDFOR
		
	\STATE	(\textit{CDU} step)\;
		\FOR{$j=1$ to $J$}
			\FOR{$c=1$ to $C$}
				\FOR{$i=1$ to $N_c^{j}$}
					\STATE $\Bigx_{c,i}^{j,(g)} = \begin{pmatrix}
					\mathbf{0}_{n^{j}_{c,i}\times K}&
					\x^{c,(g)}[n_{c,i}^j]\cdot\mathbf{I}_{K\times K}&
					\mathbf{0}_{(N-K-n^{j}_{c,i})\times K}
					\end{pmatrix}^{\text{T}}$
				\ENDFOR
			\ENDFOR
		\ENDFOR
		\STATE $\{\smlh_c^{(g)}\}_{c=1}^C=\argmin_{\{\smlh_c\}_{c=1}^C} -\sum_{j=1}^J\log p(\y^j\vert \{\smlh_c\}_{c=1}^C , \{\Bigx_{c,i}^{j,(g)}\}_{c,i,j=1})$, s.t. $\vectornorm{\smlh_c}_2=1$
\ENDFOR
\end{algorithmic}
\end{algorithm}
\begin{algorithm}
\caption{$\text{DCEA}(\y, \smlh, b)$: Forward pass of DCEA architecture.}
	\label{algo:ae}
\begin{algorithmic}
	\INPUT{$\y, \smlh, b, \alpha$}
	\OUTPUT{$ \w$}
	\FOR{$t =1$ to $T$}
		\STATE $\x_t = \prox_{\bvec}\left(\x_{t-1} + \alpha\Bigh^T\big(\y-f^{-1}\big(\Bigh\x_{t-1}\big)\big)\right)$
	\ENDFOR
	\STATE $\w = \Bigh \x_T$
	
	\end{algorithmic}
\end{algorithm}
\section{DCEA architecture}

\noindent We found that BCOMP converged in $G=5$ alternating-minimization iterations in the simulations, and $G=10$ iterations in the analyses of the real data. After convergence, the CSC step of the BCOMP can be used for inference on the test dataset, similar to using the encoder of DCEA for inference. Algorithm~\ref{algo:ae} shows the forward pass of the DCEA architecture. For notational convenience, we have dropped the superscript $j$ indexing the $J$ inputs.

\paragraph{Implementation of the DCEA encoder}

We implemented the DCEA architecture in PyTorch. In the case of 1D, we accelerate the computations performed by the DCEA encoder by replacing ISTA with its faster version FISTA~\cite{beck2009fast}. FISTA uses a momentum term to accelerate the converge of ISTA. The resulting encoder is similar to the one from~\cite{TolooshamsBahareh2019deepresidualAE}. We trained it using backpropagation with the ADAM optimizer~\cite{Kingma2014AdamAM}, on an Nvidia GPU (GeForce GTX 1060).

\paragraph{Hyperparameters used for training the DCEA architecture in using the simulated and real neural spiking data}

In these experiments, we treat $\lambda$ as hyperparameter where $b = \alpha \lambda$. $\lambda$ is tuned by grid search in the interval of $[0.1, 1.5]$. Following the grid search, we used $\lambda=0.38$ in the simulations and $\lambda=0.12$ for the real data. The DCEA encoder performs $T=250$ and $T=5{,}000$ iterations of FISTA, respectively for the simulated and for the real data. We found that such large numbers, particularly for the real data, were necessary for the encoder to produce sparse codes. We used $\alpha=0.2$ in the simulations and $\alpha=0.5$ for the real data. We used batches of size $256$ neurons in the simulations, and a single neuron per batch in the analyses of the real data.

\paragraph{Processing of the output of the DCEA encoder after training in neural spiking experiment} 

The encoder of the DCEA architecture performs $\ell_1$-regularized logistic regression using the convolutional dictionary $\Bigh$, the entries of which are highly correlated because of the convolutional structure. Suppose a binomial observation $\y^j$ is generated according to the binomial generative model with mean of $\boldmu_j=f^{-1}\big(\Bigh\x^j\big)$, where $f^{-1}(\cdot)$ is a sigmoid function. We observed that the estimate $\x^j_T$ of $\x^j$ obtained by feeding the group of observations to the
DCEA encoder is a vector whose nonzero entries are clustered around those of $\x^j$. This is depicted in black in Fig.~\ref{fig:post}, and is a well-known
issue with $\ell_1$-regularized regression with correlated dictionaries~\cite{bhaskar2013atomic}. Therefore, for the neural spiking data, after training the DCEA architecture, we processed the output of the encoder as follows
\begin{enumerate}
	\item \underline{Clustering}: We applied k-means clustering to $\x^j_T$ to identify $16$ clusters.
	\item \underline{Support identification}: For each cluster, we identified the index of the largest entry from $\x^j_T$ in the cluster. This yielded a set of indices that correspond to the estimated support of $\x^j$.
	\item \underline{Logistic regression}: We performed logistic regression using the group of observations and $\Bigh$ restricted to the support identified
	in the previous step. Note that this is a common procedure for $\ell_1$-regularized problems~\cite{discretize,Mardani2018}. This yielded a new set of codes $\x^j$ that were used to re-estimate $\Bigh$, similar to a single iteration of BCOMP.
\end{enumerate}

\noindent The outcome of these three steps is shown in red circle in the supplementary Fig.~\ref{fig:post}.

\begin{figure}[!ht]
	\centering
	\includegraphics[width=\linewidth]{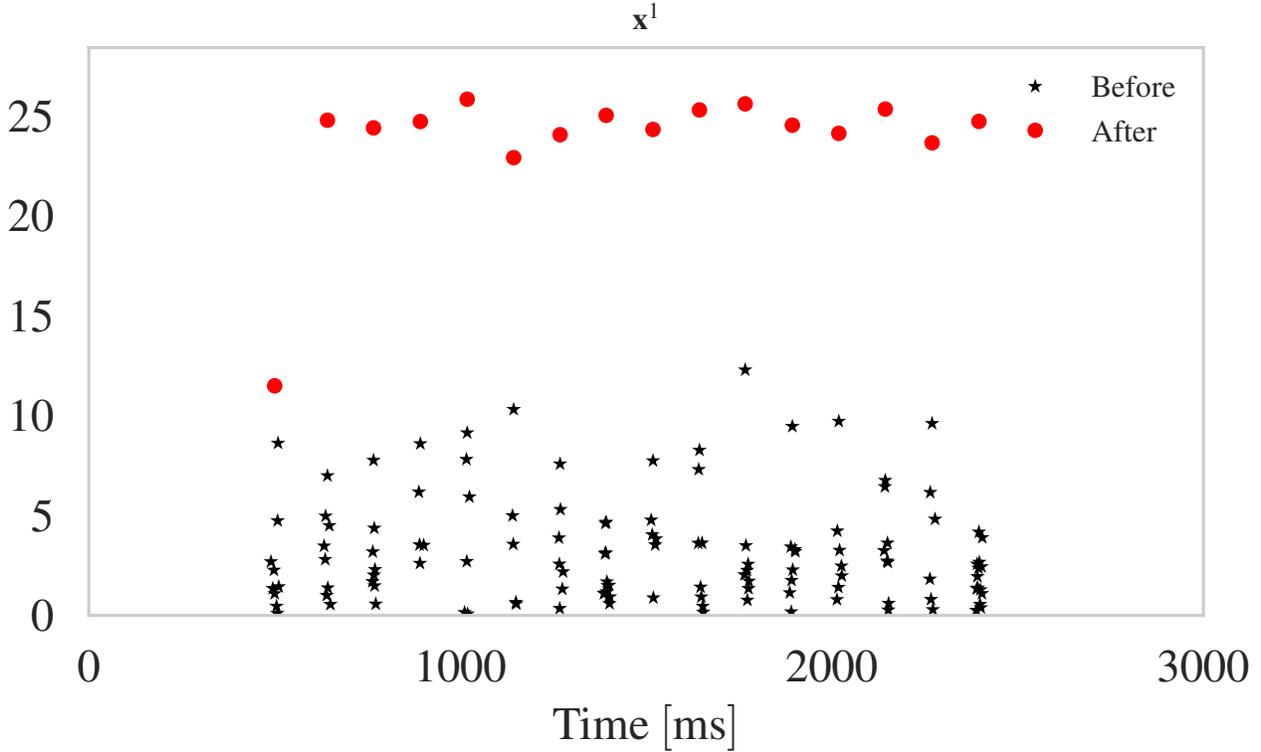}
	\vspace{-10mm}
	\caption{Output of the the DCEA encoder before and after post-processing.}
	\label{fig:post}
\end{figure}

\section{Generalized linear model (GLM) for whisker experiment}

In this section, for ease of notation, we consider the simple case of $M_j=1$ (Bernoulli). However, the detail can be generalized to the binomial generative model.\\

\noindent We describe the GLM~\cite{Truccolo2005} used for analyzing the neural spiking data from the whisker experiment \cite{whisker2014}, and which we compared to BCOMP and DCEA in Fig. 4. Fig. 4(b) depicts a segment of the periodic stimulus used in the experiment to deflect the whisker. The units are in $\frac{\text{mm}}{10}$. The full stimulus lasts $3000$ ms and is equal to zero (whisker at rest) during the two baseline periods from $0$ to $500$ ms and $2500$ to $3000$ ms. In the GLM analysis, we used whisker velocity as a stimulus covariate, which corresponds to the first difference of the position stimulus $\mathbf{s}\in\mathbb{R}^{3000}$. The blue curve in Fig. 4(c) represents one period of the whisker-velocity covariate. We associated a single stimulus coefficient $\pmb{\beta}_{\text{stim}}\in\mathbb{R}$ to this covariate. In addition to the stimulus covariate, we used history covariates in the GLM. We denote by $\pmb{\beta}_{H}^j\in\mathbb{R}^{L_j}$ the coefficients associated with these covariates, where $j=1,\cdots,J$ is the neuron index. We also define $a^j$ to be the base firing rate for neuron $j$. The GLM is given by
\begin{equation}
\begin{split}
&\y^{j}[n]\sim\text{Bernoulli}(\mathbf{p}^{j}[n])\\
&\text{ s.t. } \mathbf{p}^{j}[n] = \Big(1+\exp\big(-a^j-\pmb{\beta}_{\text{stim}}\cdot \underbrace{\left(\mathbf{s}[n]-\mathbf{s}[n-1]\right)}_{\text{whisker velocity}}-\sum_{l=1}^{L_j}\pmb{\beta}_{H}^j[l]\cdot\y^{j}[n-l]\big)\Big)^{-1}
\end{split}
\end{equation} 	
The parameters $\{a^j\}_{j=1}^J$, $\pmb{\beta}_{\text{stim}}$, and $\{\pmb{\beta}_{H}^j\}_{j=1}^J$ are estimated by minimizing the negative likelihood of the neural spiking data $\{\y^{j}\}_{j=1}^{10}$ with $M_j=30$ \emph{from all neurons} using IRLS. We picked the order $L_j$ (in ms) of the  history effect for neuron $j$ by fitting the GLM to each of the 10 neurons \emph{separately} and finding the value of $ \approx 5 \leq L_j \leq 100$ that minimizes the Akaike Information Criterion~\cite{Truccolo2005}.

\paragraph{Interpretation of the GLM as a convolutional model} Because whisker position is periodic with period $125$ ms, so is whisker velocity. Letting $\smlh_1$ denote whisker velocity in the interval of length $125$ ms starting at $500$ ms (blue curve in Fig. 4(c)), we can interpret the GLM in terms of the convolutional model of Eq. 8. In this interpretation, $\Bigh$ is the convolution matrix associated with the \emph{fixed} filter $\smlh_1$ (blue curve in Fig. 4(c)), and $\x^j$ is a sparse vector with $16$ equally spaced nonzero entries all equal to $\pmb{\beta}_{\text{stim}}$. The first nonzero entry of $\x^j$ occurs at index $500$. The number of indices between nonzero entries is $125$. The blue dots in Fig. 4(d) reflect this interpretation.

\paragraph{Incorporating history dependence in the generative model}

GLMs of neural spiking data \cite{Truccolo2005} include a constant term that models the baseline probability of spiking $a^j$, as well as a term that models the effect of spiking history. This motivates us to use the model
\begin{equation}\label{eq:logitHistory}
\begin{aligned}
\log{\frac{p(\y^{j,m} \mid \{\smlh_c\}_{c=1}^C, \x^j,\x_H^j)}{1 - p(\y^{j,m} \mid \{\smlh_c\}_{c=1}^C, \x^j,\x_H^j)}} = a^j + \Bigh \x^j + \Bigy_j \x^j_{H},
\end{aligned}
\end{equation}
where $\y^{j,m}\in\{0,1 \}^N$ refers to $m^{\text{th}}$ trial of the binomial data $\y^j$. The $n^\text{th}$ row of $\Bigy_j \in \R^{N \times L_j}$ contains the spiking history of neuron $j$ at trial $m$ from $n-L_j$ to $n$, and $\x_H^j\in\mathbb{R}^{L_j}$ are coefficients that capture the effect of spiking history on the propensity of neuron $j$ to spike. We use the same $L_j$ estimated from GLM. We estimate $a^j$ from the average firing probability during the baseline period. The addition of the history term simply results in an additional set of variables to alternate over in the alternating-minimization interpretation of ECDL. We estimate it by adding a loop around BCOMP or backpropagation through DCEA. Every iteration of this loop first assumes $\x_H^j$ are fixed. Then, it updates the filters and $\x^j$. Finally, it solves a convex optimization problem to update $\x_H^j$ given the filters and $\x^j$. In the interest of space, we do not describe this algorithm formally.

\section{Kolmogorov-smirnov plots and the time-rescaling theorem}

Loosely, the time-rescaling theorem states that rescaling the inter-spike intervals (ISIs) of the neuron using the (unknown) underlying conditional intensity function (CIF) will transform them into i.i.d. samples from an exponential random variable with rate $1$. This implies that, if we apply the CDF of an exponential random variable with rate $1$ to the rescaled ISIs, these should look like i.i.d. draws from a uniform random variable in the interval $[0,1]$. KS plots are a visual depiction of this result. They are obtained by computing the rescaled ISIs using an estimate of the underlying CIF and applying the CDF of an exponential random variable with rate $1$ to them. These are then sorted and plotted against ideal uniformly-spaced empirical quantiles from a uniform random variable in the interval $[0,1]$. The CIF that fits the data the best is the one that yields a curve that is the closest to the 45-degree diagonal. Fig. 4(e) depicts the KS plots obtained using the CIFs estimated using DCEA, BCOMP and the GLM.

\section{Image denoising}
\begin{figure}[h]
	\begin{minipage}[b]{1.0\linewidth}
		\centering
		\tikzstyle{input} = [coordinate]
		\tikzstyle{output} = [coordinate]
		\tikzstyle{pinstyle} = [pin edge={to-,thin,black}]
		\begin{tikzpicture}[auto, node distance=2cm,>=latex']
		cloud/.style={
			draw=red,
			thick,
			ellipse,
			fill=none,
			minimum height=1em}

		\node [input, name=input] {};
		
		\node [rectangle, fill=none, node distance=0.001cm, right of=input] (A) {$\includegraphics[width=0.24\linewidth]{./figures/man_clean}$};
		\node [rectangle, fill=none, node distance=4.15cm, right of=A] (B) {$\includegraphics[width=0.24\linewidth]{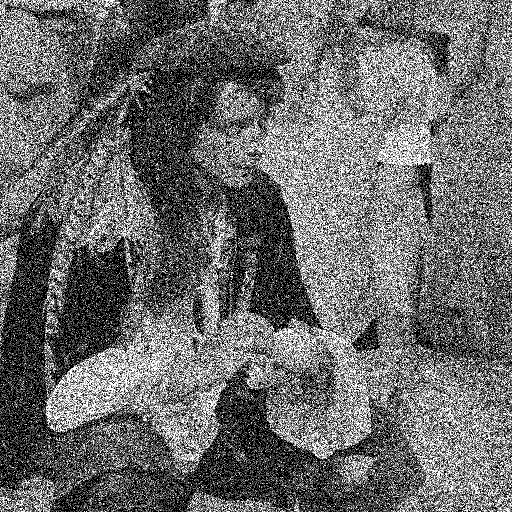}$};
		\node [rectangle, fill=none, node distance=4.15cm, right of=B] (C) {$\includegraphics[width=0.24\linewidth]{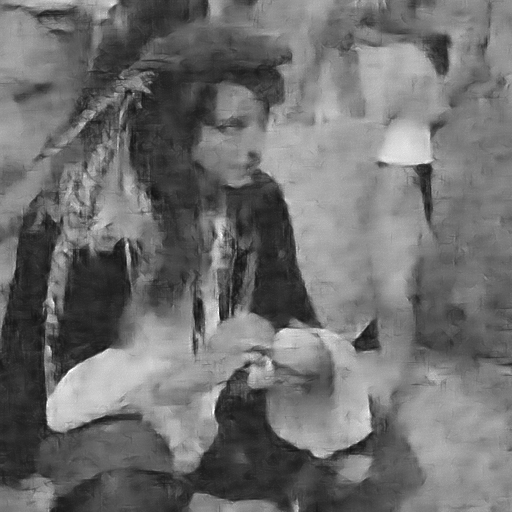}$};
		\node [rectangle, fill=none, node distance=4.15cm, right of=C] (D) {$\includegraphics[width=0.24\linewidth]{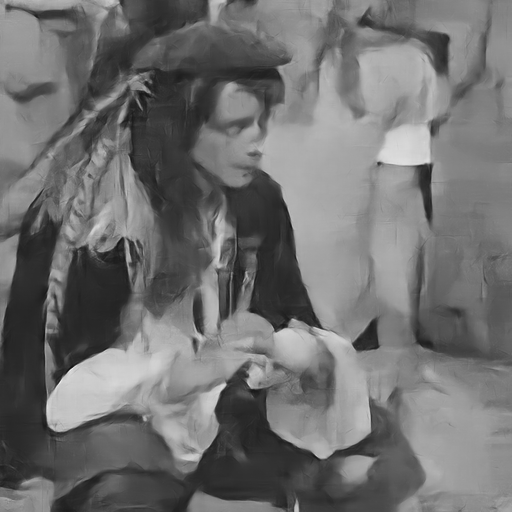}$};
				
		\node [rectangle, fill=none,  node distance=2.25cm,  above of=A] (text) {(a) Original};
		\node [rectangle, fill=none,  node distance=2.25cm,  above of=B] (text) {(b) Noisy peak$=4$};
		\node [rectangle, fill=none,  node distance=2.25cm,  above of=C] (text) {(c) DCEA-C};
		\node [rectangle, fill=none,  node distance=2.25cm,  above of=D] (text) {(d) DCEA-UC};     
	
	
		\node [rectangle, fill=none, node distance=4.15cm, below of=A] (A) {$\includegraphics[width=0.24\linewidth]{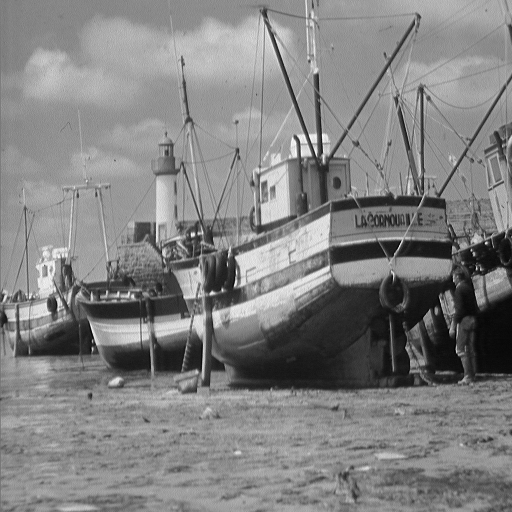}$};
		\node [rectangle, fill=none, node distance=4.15cm, right of=A] (B) {$\includegraphics[width=0.24\linewidth]{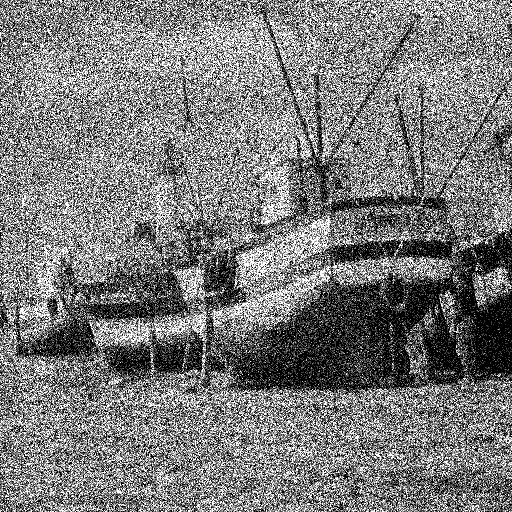}$};
		\node [rectangle, fill=none, node distance=4.15cm, right of=B] (C) {$\includegraphics[width=0.24\linewidth]{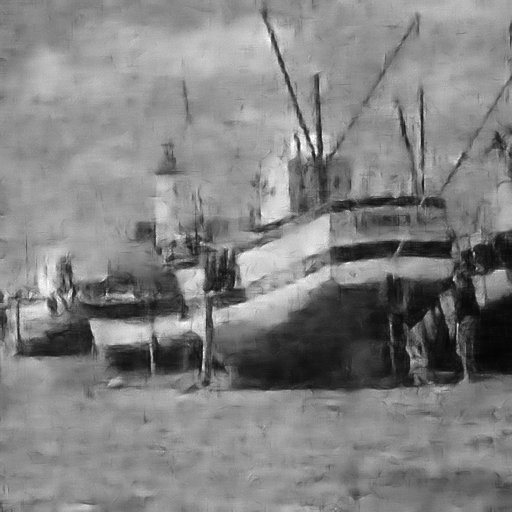}$};
		\node [rectangle, fill=none, node distance=4.15cm, right of=C] (D) {$\includegraphics[width=0.24\linewidth]{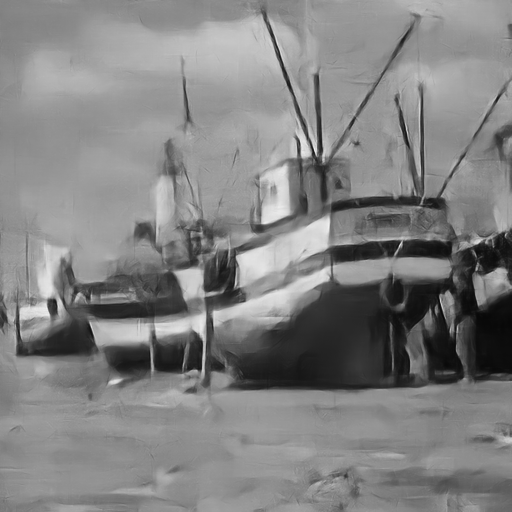}$};
			
		\node [rectangle, fill=none, node distance=4.15cm, below of=A] (A) {$\includegraphics[width=0.24\linewidth]{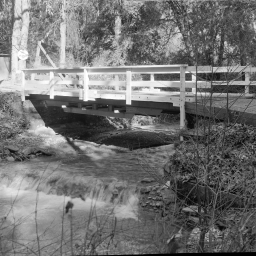}$};
		\node [rectangle, fill=none, node distance=4.15cm, right of=A] (B) {$\includegraphics[width=0.24\linewidth]{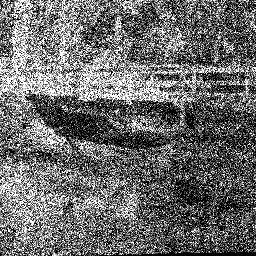}$};
		\node [rectangle, fill=none, node distance=4.15cm, right of=B] (C) {$\includegraphics[width=0.24\linewidth]{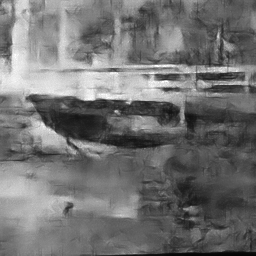}$};
		\node [rectangle, fill=none, node distance=4.15cm, right of=C] (D) {$\includegraphics[width=0.24\linewidth]{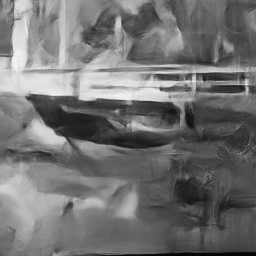}$};
		
		\node [rectangle, fill=none, node distance=4.15cm, below of=A] (A) {$\includegraphics[width=0.24\linewidth]{./figures/cameraman_clean}$};
		\node [rectangle, fill=none, node distance=4.15cm, right of=A] (B) {$\includegraphics[width=0.24\linewidth]{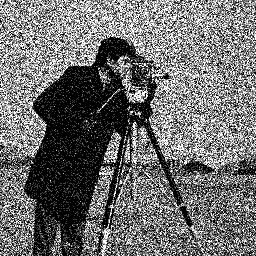}$};
		\node [rectangle, fill=none, node distance=4.15cm, right of=B] (C) {$\includegraphics[width=0.24\linewidth]{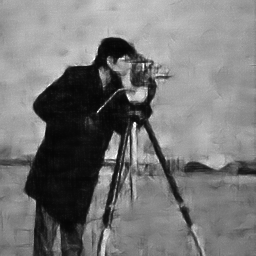}$};
		\node [rectangle, fill=none, node distance=4.15cm, right of=C] (D) {$\includegraphics[width=0.24\linewidth]{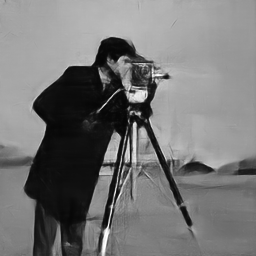}$};

		\node [rectangle, fill=none, node distance=4.15cm, below of=A] (A) {$\includegraphics[width=0.24\linewidth]{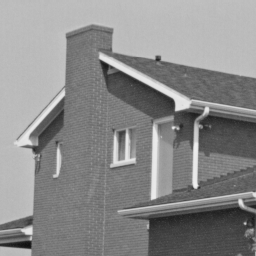}$};
		\node [rectangle, fill=none, node distance=4.15cm, right of=A] (B) {$\includegraphics[width=0.24\linewidth]{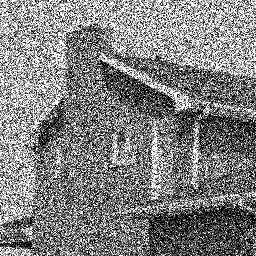}$};
		\node [rectangle, fill=none, node distance=4.15cm, right of=B] (C) {$\includegraphics[width=0.24\linewidth]{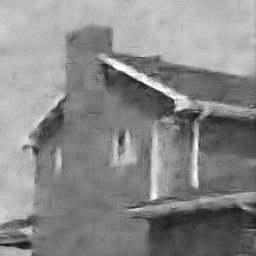}$};
		\node [rectangle, fill=none, node distance=4.15cm, right of=C] (D) {$\includegraphics[width=0.24\linewidth]{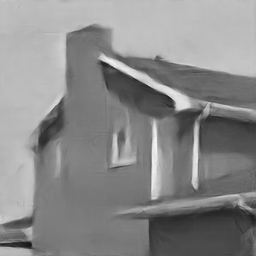}$};		
		
						
		\end{tikzpicture}
	\end{minipage}
	\vspace{-8mm}
	\caption{Denoising performance on test images with peak$=4$. (a) Original, (b) noisy, (c) DCEA-C, and (d) DCEA-UC.}
	\label{fig:house}
\end{figure}

\begin{figure}[h]
	\begin{minipage}[b]{1.0\linewidth}
		\centering
		\tikzstyle{input} = [coordinate]
		\tikzstyle{output} = [coordinate]
		\tikzstyle{pinstyle} = [pin edge={to-,thin,black}]
		\begin{tikzpicture}[auto, node distance=2cm,>=latex']
		cloud/.style={
			draw=red,
			thick,
			ellipse,
			fill=none,
			minimum height=1em}

		\node [input, name=input] {};
		
		\node [rectangle, fill=none, node distance=0.001cm, right of=input] (A) {$\includegraphics[width=0.24\linewidth]{./figures/man_clean}$};
		\node [rectangle, fill=none, node distance=4.15cm, right of=A] (B) {$\includegraphics[width=0.24\linewidth]{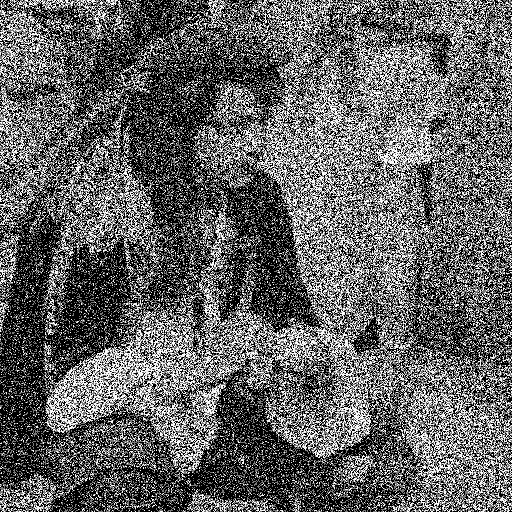}$};
		\node [rectangle, fill=none, node distance=4.15cm, right of=B] (C) {$\includegraphics[width=0.24\linewidth]{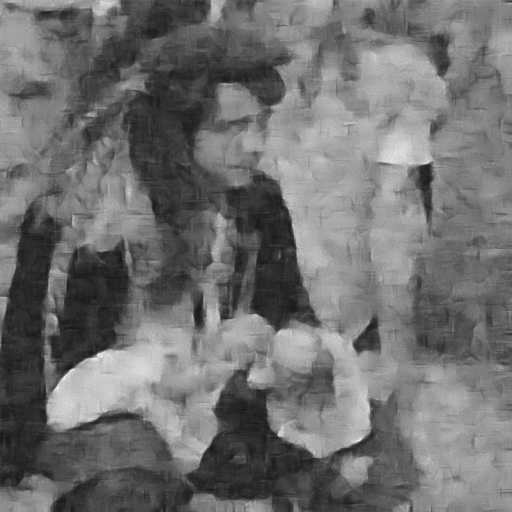}$};
		\node [rectangle, fill=none, node distance=4.15cm, right of=C] (D) {$\includegraphics[width=0.24\linewidth]{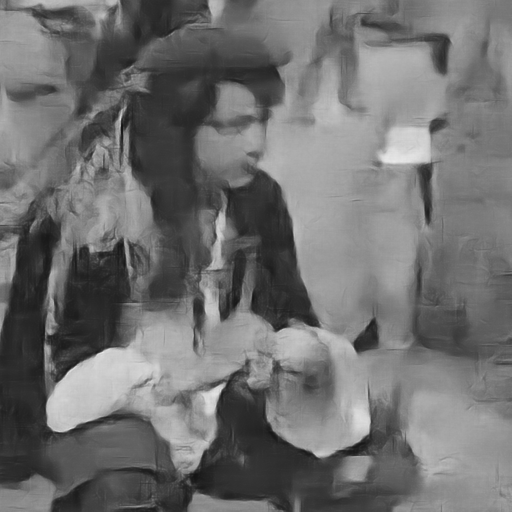}$};
				
		\node [rectangle, fill=none,  node distance=2.25cm,  above of=A] (text) {(a) Original};
		\node [rectangle, fill=none,  node distance=2.25cm,  above of=B] (text) {(b) Noisy peak$=2$};
		\node [rectangle, fill=none,  node distance=2.25cm,  above of=C] (text) {(c) DCEA-C};
		\node [rectangle, fill=none,  node distance=2.25cm,  above of=D] (text) {(d) DCEA-UC};     
	
	
		\node [rectangle, fill=none, node distance=4.15cm, below of=A] (A) {$\includegraphics[width=0.24\linewidth]{./figures/boat_clean}$};
		\node [rectangle, fill=none, node distance=4.15cm, right of=A] (B) {$\includegraphics[width=0.24\linewidth]{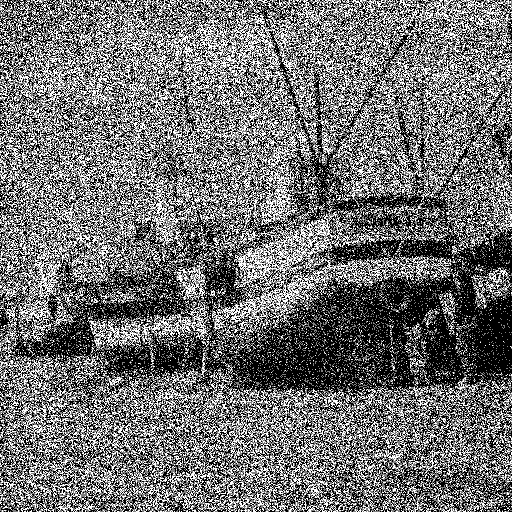}$};
		\node [rectangle, fill=none, node distance=4.15cm, right of=B] (C) {$\includegraphics[width=0.24\linewidth]{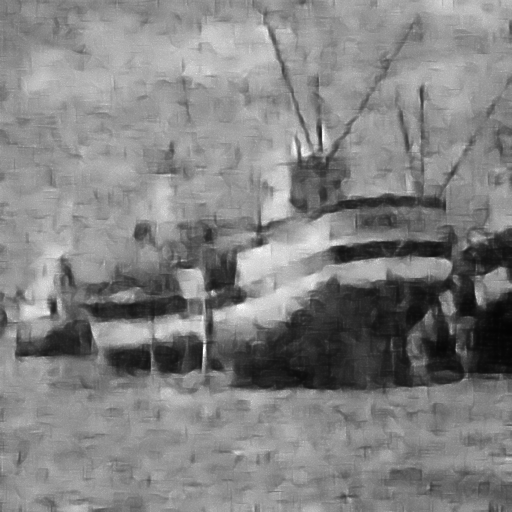}$};
		\node [rectangle, fill=none, node distance=4.15cm, right of=C] (D) {$\includegraphics[width=0.24\linewidth]{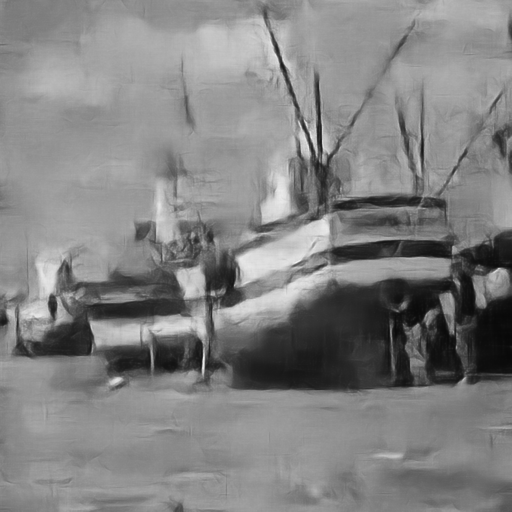}$};
			
		\node [rectangle, fill=none, node distance=4.15cm, below of=A] (A) {$\includegraphics[width=0.24\linewidth]{./figures/brdige_clean}$};
		\node [rectangle, fill=none, node distance=4.15cm, right of=A] (B) {$\includegraphics[width=0.24\linewidth]{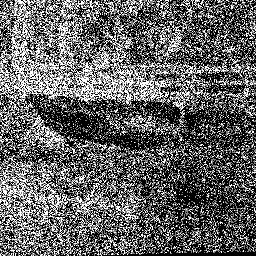}$};
		\node [rectangle, fill=none, node distance=4.15cm, right of=B] (C) {$\includegraphics[width=0.24\linewidth]{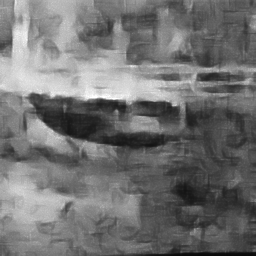}$};
		\node [rectangle, fill=none, node distance=4.15cm, right of=C] (D) {$\includegraphics[width=0.24\linewidth]{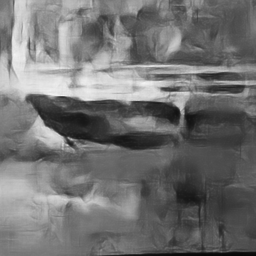}$};
		
		\node [rectangle, fill=none, node distance=4.15cm, below of=A] (A) {$\includegraphics[width=0.24\linewidth]{./figures/cameraman_clean}$};
		\node [rectangle, fill=none, node distance=4.15cm, right of=A] (B) {$\includegraphics[width=0.24\linewidth]{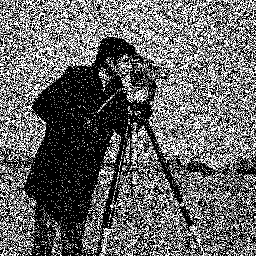}$};
		\node [rectangle, fill=none, node distance=4.15cm, right of=B] (C) {$\includegraphics[width=0.24\linewidth]{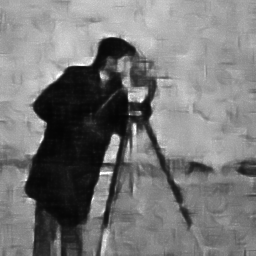}$};
		\node [rectangle, fill=none, node distance=4.15cm, right of=C] (D) {$\includegraphics[width=0.24\linewidth]{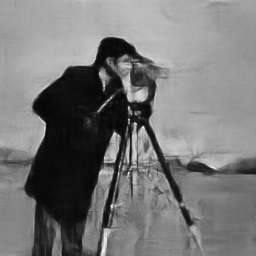}$};

		\node [rectangle, fill=none, node distance=4.15cm, below of=A] (A) {$\includegraphics[width=0.24\linewidth]{./figures/house_clean}$};
		\node [rectangle, fill=none, node distance=4.15cm, right of=A] (B) {$\includegraphics[width=0.24\linewidth]{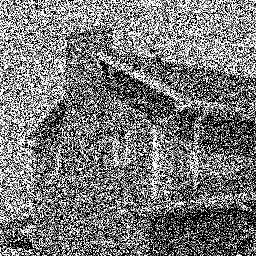}$};
		\node [rectangle, fill=none, node distance=4.15cm, right of=B] (C) {$\includegraphics[width=0.24\linewidth]{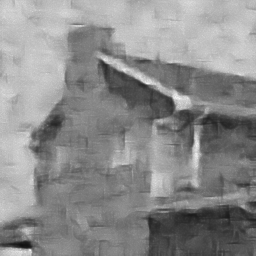}$};
		\node [rectangle, fill=none, node distance=4.15cm, right of=C] (D) {$\includegraphics[width=0.24\linewidth]{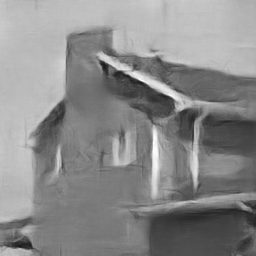}$};		
		
						
		\end{tikzpicture}
	\end{minipage}
	\vspace{-8mm}
	\caption{Denoising performance on test images with peak$=2$. (a) Original, (b) noisy, (c) DCEA-C, and (d) DCEA-UC.}
	\label{fig:house}
\end{figure}

\begin{figure}[h]
	\begin{minipage}[b]{1.0\linewidth}
		\centering
		\tikzstyle{input} = [coordinate]
		\tikzstyle{output} = [coordinate]
		\tikzstyle{pinstyle} = [pin edge={to-,thin,black}]
		\begin{tikzpicture}[auto, node distance=2cm,>=latex']
		cloud/.style={
			draw=red,
			thick,
			ellipse,
			fill=none,
			minimum height=1em}

		\node [input, name=input] {};
		
		\node [rectangle, fill=none, node distance=0.001cm, right of=input] (A) {$\includegraphics[width=0.24\linewidth]{./figures/man_clean}$};
		\node [rectangle, fill=none, node distance=4.15cm, right of=A] (B) {$\includegraphics[width=0.24\linewidth]{./figures/man_noisy_1}$};
		\node [rectangle, fill=none, node distance=4.15cm, right of=B] (C) {$\includegraphics[width=0.24\linewidth]{./figures/man_untied_1}$};
		\node [rectangle, fill=none, node distance=4.15cm, right of=C] (D) {$\includegraphics[width=0.24\linewidth]{./figures/man_untied_1}$};
				
		\node [rectangle, fill=none,  node distance=2.25cm,  above of=A] (text) {(a) Original};
		\node [rectangle, fill=none,  node distance=2.25cm,  above of=B] (text) {(b) Noisy peak$=1$};
		\node [rectangle, fill=none,  node distance=2.25cm,  above of=C] (text) {(c) DCEA-C};
		\node [rectangle, fill=none,  node distance=2.25cm,  above of=D] (text) {(d) DCEA-UC};     
	
	
		\node [rectangle, fill=none, node distance=4.15cm, below of=A] (A) {$\includegraphics[width=0.24\linewidth]{./figures/boat_clean}$};
		\node [rectangle, fill=none, node distance=4.15cm, right of=A] (B) {$\includegraphics[width=0.24\linewidth]{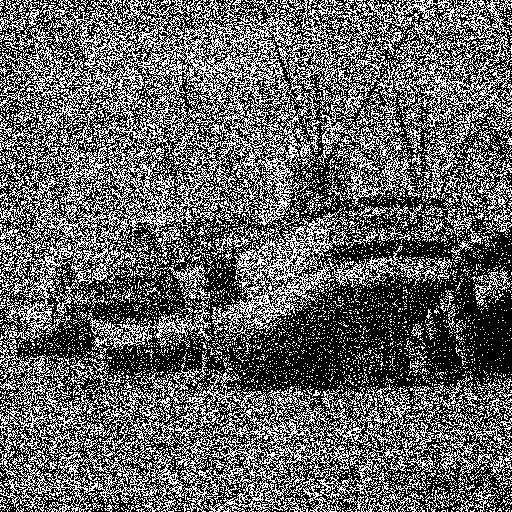}$};
		\node [rectangle, fill=none, node distance=4.15cm, right of=B] (C) {$\includegraphics[width=0.24\linewidth]{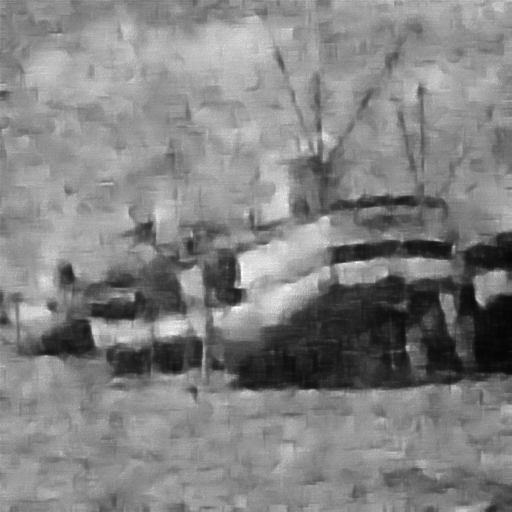}$};
		\node [rectangle, fill=none, node distance=4.15cm, right of=C] (D) {$\includegraphics[width=0.24\linewidth]{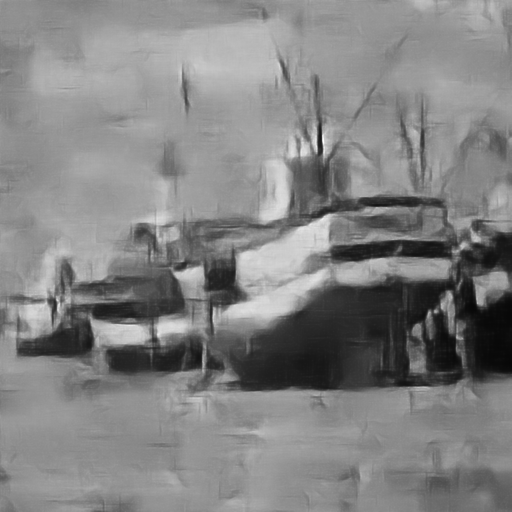}$};
			
		\node [rectangle, fill=none, node distance=4.15cm, below of=A] (A) {$\includegraphics[width=0.24\linewidth]{./figures/brdige_clean}$};
		\node [rectangle, fill=none, node distance=4.15cm, right of=A] (B) {$\includegraphics[width=0.24\linewidth]{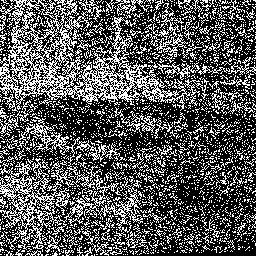}$};
		\node [rectangle, fill=none, node distance=4.15cm, right of=B] (C) {$\includegraphics[width=0.24\linewidth]{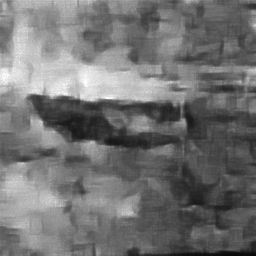}$};
		\node [rectangle, fill=none, node distance=4.15cm, right of=C] (D) {$\includegraphics[width=0.24\linewidth]{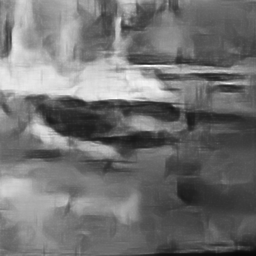}$};
		
		\node [rectangle, fill=none, node distance=4.15cm, below of=A] (A) {$\includegraphics[width=0.24\linewidth]{./figures/cameraman_clean}$};
		\node [rectangle, fill=none, node distance=4.15cm, right of=A] (B) {$\includegraphics[width=0.24\linewidth]{./figures/cameraman_noisy_1}$};
		\node [rectangle, fill=none, node distance=4.15cm, right of=B] (C) {$\includegraphics[width=0.24\linewidth]{./figures/cameraman_tied_1}$};
		\node [rectangle, fill=none, node distance=4.15cm, right of=C] (D) {$\includegraphics[width=0.24\linewidth]{./figures/cameraman_untied_1}$};

		\node [rectangle, fill=none, node distance=4.15cm, below of=A] (A) {$\includegraphics[width=0.24\linewidth]{./figures/house_clean}$};
		\node [rectangle, fill=none, node distance=4.15cm, right of=A] (B) {$\includegraphics[width=0.24\linewidth]{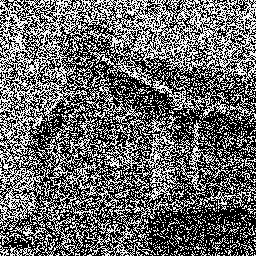}$};
		\node [rectangle, fill=none, node distance=4.15cm, right of=B] (C) {$\includegraphics[width=0.24\linewidth]{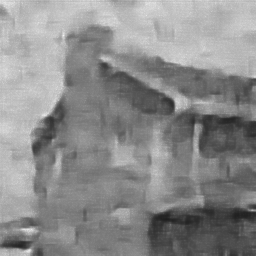}$};
		\node [rectangle, fill=none, node distance=4.15cm, right of=C] (D) {$\includegraphics[width=0.24\linewidth]{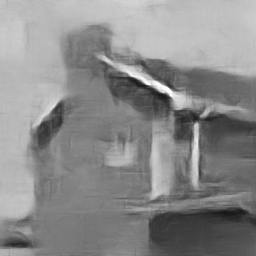}$};		
		
						
		\end{tikzpicture}
	\end{minipage}
	\vspace{-8mm}
	\caption{Denoising performance on test images with peak$=1$. (a) Original, (b) noisy, (c) DCEA-C, and (d) DCEA-UC.}
	\label{fig:house}
\end{figure}

\clearpage

\end{document}